\documentclass[journal]{IEEEtran}

\usepackage{amsmath,graphics,epsfig,amssymb,multirow,cite,hhline}
\newcommand{\argmin}{\operatornamewithlimits{argmin}}
\newcommand{\argmax}{\operatornamewithlimits{argmax}}

\newtheorem{theorem}{Theorem}[section]
\newtheorem{lemma}[theorem]{Lemma}

\newenvironment{proof}[1][Proof]{\begin{trivlist}
\item[\hskip \labelsep {\bfseries #1}]}{\end{trivlist}}

\newcommand{\qed}{\nobreak \ifvmode \relax \else
   \ifdim\lastskip<1.5em \hskip-\lastskip
   \hskip1.5em plus0em minus0.5em \fi \nobreak
   \vrule height0.75em width0.5em depth0.25em\fi}

\begin{document}
\title{Learning Stable Multilevel Dictionaries for\\
Sparse Representations}

\author{Jayaraman~J.~Thiagarajan, 
    Karthikeyan~Natesan~Ramamurthy
    and~Andreas~Spanias
\thanks{The authors are with the SenSIP Center, School of Electrical, Computer and Energy Engineering, Arizona State University, Tempe, AZ, 85287.}\\
E-mail: \{jjayaram,knatesan,spanias\}@asu.edu.
}

\markboth{Learning Stable Multilevel Dictionaries for Sparse Representations}%
{Shell \MakeLowercase{\textit{et al.}}: Learning Stable Multilevel Dictionaries Sparse Representations}

\maketitle

\begin{abstract}
Sparse representations using learned dictionaries are being increasingly used with success in several data processing and machine learning applications. The availability of abundant training data necessitates the development of efficient, robust and provably good dictionary learning algorithms.  Algorithmic stability and generalization are desirable characteristics for dictionary learning algorithms that aim to build global dictionaries which can efficiently model any test data similar to the training samples. In this paper, we propose an algorithm to learn dictionaries for sparse representations from large scale data, and prove that the proposed learning algorithm is stable and generalizable asymptotically. The algorithm employs a 1-D subspace clustering procedure, the K-hyperline clustering, in order to learn a hierarchical dictionary with multiple levels. We also propose an information-theoretic scheme to estimate the number of atoms needed in each level of learning and develop an ensemble approach to learn robust dictionaries. Using the proposed dictionaries, the sparse code for novel test data can be computed using a low-complexity pursuit procedure. We demonstrate the stability and generalization characteristics of the proposed algorithm using simulations. We also evaluate the utility of the multilevel dictionaries in compressed recovery and subspace learning applications.
\end{abstract}

\IEEEpeerreviewmaketitle

\section{Introduction}
\subsection{Dictionary Learning for Sparse Representations}
\label{sec:sparse_rep_intro}
\IEEEPARstart{S}{everal} types of naturally occurring data have most of their energy concentrated in a small number of features when represented using an linear model. In particular, it has been shown that the statistical structure of naturally occurring signals and images allows for their efficient representation as a sparse linear combination of elementary features \cite{Field1994}. A finite collection of normalized features is referred to as a dictionary. The linear model used for general sparse coding is given by
\begin{equation}
\mathbf{y} = \boldsymbol{\Psi}\mathbf{a}+\mathbf{n},
\label{eqn:basic}
\end{equation}where $\mathbf{y} \in \mathbb{R}^M$ is the data vector and $\mathbf{\Psi} = [\boldsymbol{\psi}_1 \boldsymbol{\psi}_2 \ldots \boldsymbol{\psi}_K] \in \mathbb{R}^{M \times K}$ is the dictionary. Each column of the dictionary, referred to as an atom, is a representative pattern normalized to unit $\ell_2$ norm. $\mathbf{a} \in \mathbb{R}^{K}$ is the sparse coefficient vector and $\mathbf{n}$ is a noise vector whose elements are independent realizations from the Gaussian distribution $\mathcal{N}(0,\sigma^2)$.

The sparse coding problem is usually solved as
\begin{equation}
\hat{\mathbf{a}} = \argmin_\mathbf{a} \|\mathbf{a}\|_0 \text{ subj. to } \|\mathbf{y} - \boldsymbol{\Psi}\mathbf{a}\|_2^2 \leq \epsilon,
\label{eqn:sc_l0}
\end{equation} where $\|.\|_0$ indicates the $\ell_0$ sparsity measure which counts the number of non-zero elements, $\|.\|_2$ denotes the $\ell_2$ norm and $\epsilon$ is the error goal for the representation. The $\ell_1$ norm, denoted by $\|.\|_1$, can be used instead of $\ell_0$ measure to convexify (\ref{eqn:sc_l0}). A variety of methods can be found in the literature to obtain sparse representations efficiently  \cite{tropp,Chen_BP, Elad_2007_IS, lars}. The sparse coding model has been successfully used for inverse problems in images \cite{Elad_KSVD}, and also in machine learning applications such as classification, clustering, and subspace learning to name a few \cite{Aviyente2006, JT_radar,wright,Ramirez,scspm,Yu,zhang,LC-KSVD,JT_SLSC,l1graph}. 

The dictionary $\mathbf{\Psi}$ used in (\ref{eqn:sc_l0}) can be obtained from predefined bases, designed from a union of orthonormal bases \cite{Gribonval_Nielsen_2003}, or structured as an overcomplete set of individual vectors optimized to the data \cite{Lewicki_2000}. A wide range of batch and online dictionary learning algorithms have been proposed in the literature \cite{Mairal_Bach_Ponce_Sapiro_2009, Jenatton_Mairal_Obozinski_Bach_2010, Bar2010,Rubin2010, tosic2011,mairal2012task, zhou2012locality,wang2012supervised,ramirez2012mdl}, some of which are tailored for specific tasks. The conditions under which a dictionary can be identified from the training data using an $\ell_1$ minimization approach are derived in \cite{Gribonval2010}. The joint optimization problem for dictionary learning and sparse coding can be expressed as \cite{Elad_KSVD}
\begin{equation}
\min_{\mathbf{\Psi},\mathbf{A}} \|\mathbf{Y} - \mathbf{\Psi} \mathbf{A}\|_F^2 \text{ subj. to } \|\mathbf{a}_i\|_0 \leq S, \forall i, \|\boldsymbol{\psi}_j\|_2 = 1, \forall j,
\label{eqn:sparse_rep_dict_learn}
\end{equation} where $\mathbf{Y} = \left[\mathbf{y}_1 \mathbf{y}_2 \ldots \mathbf{y}_T \right]$ is a matrix of $T$ training vectors, $\mathbf{A} = \left[\mathbf{a}_1 \mathbf{a}_2 \ldots \mathbf{a}_T \right]$ is the coefficient matrix, $S$ is the sparsity of the coefficient vector and $\|.\|_F$ denotes the Frobenius norm. 

\subsection{Multilevel Learning}
\label{sec:MLD_intro}
In this paper, we propose a hierarchical multilevel dictionary learning algorithm that is implicitly regularized to aid in sparse approximation of data. The proposed multilevel dictionary (MLD) learning algorithm is geared towards obtaining \textit{global} dictionaries for the entire probability space of the data, which are \textit{provably} stable, and generalizable to novel test data. In addition, our algorithm involves simple schemes for learning and representation: a 1-D subspace clustering algorithm (\textit{K-hyperline clustering} \cite{Cichoki2009}) is used to infer atoms in each level, and $1-$sparse representations are obtained in each level using a pursuit scheme that employs just \textit{correlate-and-max} operations. In summary, the algorithm creates a sub-dictionary for each level and obtains a residual which is used as the training data for the next level, and this process is continued until a pre-defined stopping criterion is reached.

The primary utility of sparse models with learned dictionaries in data processing and machine learning applications stems from the fact that the dictionary atoms serve as \textit{predictive features}, capable of providing a good representation for some aspect of the test data. From the viewpoint of statistical learning theory \cite{poggio2004}, a good predictive model is one that is stable and generalizable, and MLD learning satisfies both these properties. To the best of our knowledge, there is no other dictionary learning method which has been proven to satisfy these properties. Generalization ensures that the learned dictionary can successfully represent test data drawn from the same probability space as the training data, and stability guarantees that it is possible to reliably learn such a dictionary from an arbitrary training set. In other words, the asymptotic stability and generalization of MLD provides theoretical justification for the uniformly good performance of global multilevel dictionaries. We can minimize the risk of overfitting further by choosing a proper model order. We propose a method based on the minimum description length (MDL) principle \cite{grunwald2005advances} to choose the optimal model order, which in our case corresponds to the number of dictionary elements in each level. Recently, other approaches have been proposed to choose the best order for a given sparse model using MDL \cite{ramirez2012mdl}, so that the generalization error is minimized. However, the difference in our case is that, in addition to optimizing the model order for a given training set using MDL, we prove that \textit{any} dictionary learned using MLD is generalizable and stable. Since both generalization and stability are asymptotic properties, we also propose a robust variant of our MLD algorithm using randomized ensemble methods, to obtain an improved performance with test data. Note that our goal is not to obtain dictionaries optimized for a specific task \cite{mairal2012task}, but to propose a general predictive sparse modeling framework that can be suitably adapted for any task.

The dictionary atoms in MLD are structurally regularized, and therefore the hierarchy in representation is imposed implicitly for the novel test data, leading to improved recovery in ill-posed and noise-corrupted problems. Considering dictionary learning with image patches as an example, in MLD the predominant atoms in the first few levels (see Figure \ref{Fig:dict_BSDS}) always contribute the highest energy to the representation. For natural image data, it is known that the patches are comprised of geometric patterns or stochastic textures or a combination of both \cite{Zhu2010}. Since the geometric patterns usually are of higher energy when compared to stochastic textures in images, MLD learns the geometric patterns in the first few levels and stochastic textures in the last few levels, thereby adhering to the natural hierarchy in image data. The hierarchical multistage vector quantization (MVQ) \cite{Gray} is related to MLD learning. The important difference, however, is that dictionaries obtained for sparse representations must assume that the data lies in a union-of-subspaces, and the MVQ does not incorporate this assumption. Note that multilevel learning is also different from the work in \cite{Yu_2010}, where multiple sub-dictionaries are designed and one of them is chosen for representing a group of patches.

\subsection{Stability and Generalization in Learning}
A learning algorithm is a map from the space of training examples to the hypothesis space of functional solutions. In clustering, the learned function is completely characterized by the cluster centers. Stability of a clustering algorithm implies that the cluster centroids learned by the algorithm are not significantly different when different sets of i.i.d. samples from the same probability space are used for training \cite{kmeansRakhlin2007}. When there is a unique minimizer to the clustering objective with respect to the underlying data distribution, stability of a clustering algorithm is guaranteed \cite{BenDavid2006} and this analysis has been extended to characterize the stability of K-means clustering in terms of the number of minimizers \cite{BenDavid2007}. In \cite{JT_khyp}, the stability properties of the K-hyperline clustering algorithm have been analyzed and they have been shown to be similar to those of K-means clustering. Note that all the stability characterizations depend only on the underlying data distribution and the number of clusters, and not on the actual training data itself. Generalization implies that the average empirical training error becomes asymptotically close to the expected error with respect to the probability space of data. In \cite{Maurer2010}, the generalization bound for sparse coding in terms of the number of samples $T$, also referred to as sample complexity, is derived and in \cite{Vainsencher2011} the bound is improved by assuming a class of dictionaries that are nearly orthogonal. Clustering algorithms such as the K-means and the K-hyperline can be obtained by constraining the desired sparsity in (\ref{eqn:sparse_rep_dict_learn}) to be $1$. Since the stability characteristics of clustering algorithms are well understood, employing similar tools to analyze a general dictionary learning framework such as MLD can be beneficial. 


\subsection{Contributions}
\label{sec:contrib}
In this paper, we propose the MLD learning algorithm to design global representative dictionaries for image patches. We show that, for a sufficient number of levels, the proposed algorithm converges, and also demonstrate that a multilevel dictionary with a sufficient number of atoms per level exhibits energy hierarchy (Section \ref{sec:convergence}). Furthermore, in order to estimate the number of atoms in each level of MLD, we provide an information-theoretic approach based on the MDL principle (Section \ref{sec:mdl_lev_est}). In order to compute sparse codes for test data using the proposed dictionary, we develop the simple Multilevel Pursuit (MulP) procedure and quantify its computational complexity (Section \ref{sec:sparse_approx_MLD}). We also propose a method to obtain robust dictionaries with limited training data using ensemble methods (Section \ref{sec:robust_MLD}). Some preliminary algorithmic details and results obtained using MLD have been reported in \cite{JT_MLD}.

Using the fact that the K-hyperline clustering algorithm is stable, we perform stability analysis of the MLD algorithm. For any two sets of i.i.d. training samples from the same probability space, as the number of training samples $T \rightarrow \infty$, we show that the dictionaries learned become close to each other asymptotically. When there is a unique minimizer to the objective in each level of learning, this holds true even if the training sets are completely disjoint. However, when there are multiple minimizers for the objective in at least one level, we prove that the learned dictionaries are asymptotically close when the difference between their corresponding training sets is $o(\sqrt{T})$. Instability of the algorithm when the difference between two training sets is $\Omega(\sqrt{T})$, is also shown for the case of multiple minimizers (Section \ref{sec:stability}-C). Furthermore, we prove the asymptotic generalization of the learning algorithm (Section \ref{sec:gen_proof}). 

In addition to demonstrating the stability and the generalization behavior of MLD learning with image data (Sections \ref{sec:stab} and \ref{sec:gen}), we evaluate its performance in compressed recovery of images (Section \ref{sec:comp_rec}). Due to its theoretical guarantees, the proposed MLD effectively recovers novel test images from severe degradation (random projection). Interestingly, the proposed greedy pursuit with robust multilevel dictionaries results in improved recovery performance when compared to $\ell_1$ minimization with online dictionaries, particularly at reduced number of measurements and in presence of noise. Furthermore, we perform subspace learning with graphs constructed using sparse codes from MLD and evaluate its performance in classification (Section \ref{sec:subspace}). We show that the proposed approach outperforms subspace learning with neighborhood graphs as well as graphs based on sparse codes from conventional dictionaries.

\section{Background}
\label{sec:background}
In this section, we describe the K-hyperline clustering, a 1-D subspace clustering procedure proposed in \cite{Cichoki2009}, which forms a building block of the proposed dictionary learning algorithm. Furthermore, we briefly discuss the results for stability analysis of K-means and K-hyperline algorithms reported in \cite{kmeansRakhlin2007} and \cite{JT_khyp} respectively. The ideas described in this section will be used in Section \ref{sec:stability} to study the stability characteristics of the proposed dictionary learning procedure.

\subsection{K-hyperline Clustering Algorithm}
\label{sec:k_hyp_clus_alg}
The K-hyperline clustering algorithm is an iterative procedure that performs a least squares fit of $K$ $1$-D linear subspaces to the training data \cite{Cichoki2009}. Note that the K-hyperline clustering is a special case of general subspace clustering methods proposed in \cite{Agarwal2004,Tseng2000}, when the subspaces are $1-$dimensional and constrained to pass through the origin. In contrast with K-means, K-hyperline clustering allows each data sample to have an arbitrary coefficient value corresponding to the centroid of the cluster it belongs to. Furthermore, the cluster centroids are normalized to unit $\ell_2$ norm. Given the set of $T$ data samples $\mathbf{Y} = \{\mathbf{y}_i\}_{i=1}^T$ and the number of clusters $K$, K-hyperline clustering proceeds in two stages after initialization: the cluster assignment and the cluster centroid update. In the cluster assignment stage, training vector $\mathbf{y}_i$ is assigned to a cluster $j$ based on the minimum distortion criteria, $\mathcal{H}(\mathbf{y}_i) = \argmin_j d(\mathbf{y}_i,\boldsymbol{\psi}_j)$, where the distortion measure is
\begin{equation}
d(\mathbf{y},\boldsymbol{\psi}) = \|\mathbf{y} - \boldsymbol{\psi} (\mathbf{y}^T \boldsymbol{\psi})\|_2^2.
\label{eqn:clust_dist}
\end{equation}In the cluster centroid update stage, we perform singular value decomposition (SVD) of $\mathbf{Y}_j = [\mathbf{y}_i]_{i \in \mathcal{C}_j}$, where $\mathcal{C}_j = \{i|\mathcal{H}(\mathbf{y}_i) = j\}$ contains indices of training vectors assigned to the cluster $j$. The cluster centroid is updated as the left singular vector corresponding to the largest singular value of the decomposition. This can also be computed using a linear iterative procedure. At iteration $t+1$, the $j^\text{th}$ cluster centroid is given by
\begin{equation}
\boldsymbol{\psi}_j^{(t+1)} = \displaystyle \mathbf{Y}_j \mathbf{Y}_j^T \boldsymbol{\psi}_j^{(t)} / \|\mathbf{Y}_j \mathbf{Y}_j^T \boldsymbol{\psi}_j^{(t)}\|_2.
\label{eqn:iter_clust_cent}
\end{equation} Usually a few iterations are sufficient to obtain the centroids with good accuracy. 


\subsection{Stability Analysis of Clustering Algorithms}
\label{sec:stability_intro}
Analyzing the stability of unsupervised clustering algorithms can be valuable in terms of understanding their behavior with respect to perturbations in the training set. These algorithms extract the underlying structure in the training data and the quality of clustering is determined by an accompanying cost function. As a result, any clustering algorithm can be posed as an Empirical Risk Minimization (ERM) procedure, by defining a hypothesis class of loss functions to evaluate the possible cluster configurations and to measure their quality \cite{Caponnetto2006}. For example, K-hyperline clustering can be posed as an ERM problem over the distortion function class 
\begin{equation}
\mathcal{G}_K=\left\{g_{\boldsymbol{\Psi}}(\mathbf{y})= d(\mathbf{y},\boldsymbol{\psi}_j), j=\argmax_{l \in \{1, \cdots, K\}} |\mathbf{y}^T \boldsymbol{\psi}_l| \right\}.
\label{eqn:dist_func_class}
\end{equation} The class $\mathcal{G}_K$ is constructed with functions $g_{\boldsymbol{\Psi}}$ corresponding to all possible combinations of $K$ unit length vectors from the $\mathbb{R}^M$ space for the set $\mathbf{\Psi}$. Let us define the probability space for the data in $\mathbb{R}^M$ as $(\mathcal{Y}, \boldsymbol{\Sigma},P)$, where $\mathcal{Y}$ is the sample space and $\boldsymbol{\Sigma}$ is a sigma-algebra on $\mathcal{Y}$, i.e., the collection of subsets of $\mathcal{Y}$ over which the probability measure $P$ is defined. The training samples, $\{\mathbf{y}_i\}_{i=1}^T$, are i.i.d. realizations from this space. 

Ideally, we are interested in computing the cluster centroids $\hat{\mathbf{\Psi}}$ that minimize the expected distortion $\mathbb{E}[g_{\boldsymbol{\Psi}}]$ with respect to the probability measure $P$. However, the underlying distribution of the data samples is not known and hence we resort to minimizing the average empirical distortion with respect to the training samples $\{\mathbf{y}_i\}_{i=1}^T$ as
\begin{equation}
g_{\hat{\mathbf{\Psi}}} = \argmin_{g \in \mathcal{G}_K} \frac{1}{T} \sum_{i=1}^T g_{\boldsymbol{\Psi}}(\mathbf{y}_i).
\label{eqn:best_func}
\end{equation} When the empirical averages of the distortion functions in $\mathcal{G}_K$ uniformly converge to the expected values over all probability measures $P$,
\begin{equation}
\lim_{T \rightarrow \infty} \sup_{P} \mathbb{P}\left(\sup_{g_{\boldsymbol{\Psi}} \in \mathcal{G}_K}\left|\mathbb{E}[g_{\boldsymbol{\Psi}}] - \frac{1}{T} \sum_{i=1}^T g_{\boldsymbol{\Psi}}(\mathbf{y}_i)\right| > \delta \right) = 0,
\label{eqn:donsker}
\end{equation}for any $\delta > 0$, we refer to the class $\mathcal{G}_K$ as uniform Glivenko-Cantelli (uGC). In addition, if the class also satisfies a version of the central limit theorem, it is defined as uniform Donsker \cite{Caponnetto2006}.  In order to determine if $\mathcal{G}_K$ is uniform Donsker, we have to verify if the covering number of $\mathcal{G}_K$ with respect to the supremum norm, $N_{\infty}(\gamma,\mathcal{G}_K)$, grows polynomially in the dimensions $M$ \cite{kmeansRakhlin2007}. Here, $\gamma$ denotes the maximum $L_{\infty}$ distance between an arbitrary distortion function in  $\mathcal{G}_K$, and the function that covers it. For K-hyperline clustering, the covering number is upper bounded by \cite[Lemma 2.1]{JT_khyp}
\begin{equation}
N_{\infty}(\gamma,\mathcal{G}_K) \leq \left(\frac{8R^3 K +\gamma}{\gamma}\right)^{MK},
\label{eqn:cover_num}
\end{equation} where we assume that the data lies in an $M$-dimensional $\ell_2$ ball of radius $R$ centered at the origin.  Therefore, $\mathcal{G}_K$ belongs to the uniform Donsker class.

Stability implies that the algorithm should produce cluster centroids that are not significantly different when different i.i.d. sets from the same probability space are used for training \cite{BenDavid2006,BenDavid2007,kmeansRakhlin2007}. Stability is characterized based on the number of minimizers to the clustering objective with respect to the underlying data distribution. A minimizer corresponds to a function $g_{\mathbf{\Psi}} \in \mathcal{G}_K$ with the minimum expectation $\mathbb{E}[g_{\mathbf{\Psi}}]$. Stability analysis of K-means clustering has been reported in \cite{BenDavid2007,kmeansRakhlin2007}. Though the geometry of K-hyperline clustering is different from that of K-means, the stability characteristics of the two algorithms have been found to be similar \cite{JT_khyp}. 

Given two sets of cluster centroids $\boldsymbol{\Psi} = \{\boldsymbol{\psi}_1 ,\ldots,\boldsymbol{\psi}_K\}$ and $\boldsymbol{\Lambda}=\{\boldsymbol{\lambda}_1,\ldots,\boldsymbol{\lambda}_K\}$ learned from training sets of $T$ i.i.d. samples each realized from the same probability space, let us define the $L_1(P)$ distance between the clusterings as
\begin{equation}
\|g_{\mathbf{\Psi}} - g_{\mathbf{\Lambda}}\|_{L_1(P)} = \int |g_{\mathbf{\Psi}}(\mathbf{y}) - g_{\mathbf{\Lambda}}(\mathbf{y})| dP(\mathbf{y}).
\label{eqn:dist_clustering}
\end{equation}When $T \rightarrow \infty$, and $\mathcal{G}_K$ is uniform Donsker, stability in terms of the distortion functions  is expressed as
\begin{equation}
\|g_{\mathbf{\Psi}} - g_{\mathbf{\Lambda}}\|_{L_1(P)} \xrightarrow{P} 0,
\label{eqn:dist_clustering_stable}
\end{equation} where $\xrightarrow{P}$ denotes convergence in probability. This holds true even for $\mathbf{\Psi}$ and $\mathbf{\Lambda}$ learned from completely disjoint training sets, when there is a unique minimizer to the clustering objective. When there are multiple minimizers, (\ref{eqn:dist_clustering_stable}) holds true with respect to a change in $o(\sqrt{T})$ samples between two training sets and fails to hold with respect to a change in  $\Omega(\sqrt{T})$ samples \cite{JT_khyp}. The distance between the cluster centroids themselves is defined as \cite{kmeansRakhlin2007} 
\begin{equation}
\Delta(\mathbf{\Psi},\mathbf{\Lambda})=\max_{1 \leq j \leq K} \min_{1 \leq l \leq K} \left[(d(\boldsymbol{\psi}_j,\boldsymbol{\lambda}_l))^{1/2} + (d(\boldsymbol{\psi}_l,\boldsymbol{\lambda}_j))^{1/2}\right].
\label{eqn:dist_clus_cent}
\end{equation}
\begin{lemma}[\cite{JT_khyp}]
If the $L_1(P)$ distance between the distortion functions for the clusterings $\mathbf{\Psi}$ and $\mathbf{\Lambda}$ is bounded as $\|g_\mathbf{\Psi} - g_\mathbf{\Lambda}\|_{L_1(P)} < \mu$, for some $\mu > 0$, and $dP(\mathbf{y})/d\mathbf{y} > C$, for some $C>0$, then $\Delta(\mathbf{\Psi},\mathbf{\Lambda}) \leq 2 \sin (\rho)$ where
\begin{equation}
\rho \leq 2 \sin^{-1}\left[\frac{1}{r} \left(\frac{\mu}{\hat{C}_{C,M}}\right)^{\frac{1}{M+1}} \right].
\label{eqn:rho_bound}
\end{equation} Here the training data is assumed to lie outside an $M$-dimensional $\ell_2$ ball of radius $r$ centered at the origin, and the constant $\hat{C}_{C,M}$ depends only on $C$ and $M$.
\label{thm:dist_clus_cent}
\end{lemma} When the clustering algorithm is stable according to (\ref{eqn:dist_clustering_stable}), for admissible values of $r$, Lemma \ref{thm:dist_clus_cent} shows that the cluster centroids become arbitrarily close to each other.

\begin{table}[t]
\caption{Algorithm for building a multilevel dictionary.}
\centering
\begin{tabular}{|l|}
\hline
\textbf{Input} \\
$\mathbf{Y} = \left[ \mathbf{y}_i \right]_{i=1}^T$, $M \times T$ matrix of training vectors. \\
$L$, maximum number of levels of the dictionary.\\
$K_l$, number of dictionary elements in level $l$, $l = \{1,2,...,L\}$.\\
$\epsilon$, error goal of the representation.\\
\\
\textbf{Output} \\
$\mathbf{\Psi}_l$, adapted sub-dictionary for level $l$.\\
\\
\textbf{Algorithm}\\
Initialize: $l=1$ and $\mathbf{R}_0 = \mathbf{Y}$.\\
$\Lambda_0= \{i\text{ }|\text{ }\|\mathbf{r}_{0,i}\|_2^2 > \epsilon, 1 \leq i \leq T\}$, index of training vectors with \\ squared norm greater than error goal.\\
$\hat{\mathbf{R}}_0 = \left[ \mathbf{r}_{0,i} \right]_{i \in \Lambda_0}.$\\
\\
\textbf{while} $\Lambda_{l-1} \neq \emptyset$ and $l \leq L$\\
\quad Initialize: \\
\quad \quad $\mathbf{A}_l$, coefficient matrix, size $K_l \times M$, all zeros.\\
\quad \quad $\mathbf{R}_l$, residual matrix for level $l$, size $M \times T$, all zeros.\\
\quad $\{\mathbf{\Psi}_l,\hat{\mathbf{A}}_l \} = \text{KLC}(\hat{\mathbf{R}}_{l-1},K_l)$.\\
\quad $\mathbf{R}_l^t = \hat{\mathbf{R}}_{l-1} - \mathbf{\Psi}_l \hat{\mathbf{A}}_l$.\\
\quad $\mathbf{r}_{l,i} = \mathbf{r}_{l,j}^t \text{ where } i = \Lambda_{l-1}(j), \text{ } \forall j = {1,...,|\Lambda_{l-1}|}$.\\
\quad $\mathbf{a}_{l,i} = \hat{\mathbf{a}}_{l,j} \text{ where } i = \Lambda_{l-1}(j), \text{ } \forall j = {1,...,|\Lambda_{l-1}|}$.\\
\quad $\Lambda_l= \{i\text{ }|\text{ }\|\mathbf{r}_{l,i}\|_2^2 > \epsilon, 1 \leq i \leq T\}$.\\
\quad $\hat{\mathbf{R}}_l = \left[ \mathbf{r}_{l,i} \right]_{i \in \Lambda_l}$.\\
\quad $l \leftarrow l+1$.\\
\textbf{end}\\
\hline
\end{tabular}
\label{Table:Algorithm_MLD}
\end{table}

\section{Multilevel Dictionary Learning}
\label{sec:multilevel_dict}
In this section, we develop the multilevel dictionary learning algorithm, whose algorithmic stability and generalizability will be proved in Section \ref{sec:stability}. Furthermore, we propose strategies to estimate the number of atoms in each level and make the learning process robust for improved generalization. We also present a simple pursuit scheme to compute representations for novel test data using the MLD.

\subsection{Algorithm}
\label{sec:multilevel_dict_learning}

We denote the MLD as $\mathbf{\Psi} = [\mathbf{\Psi}_1 \mathbf{\Psi}_2 ... \mathbf{\Psi}_L]$, and the coefficient matrix as $\mathbf{A} = [\mathbf{A}_1^T \mathbf{A}_2^T ... \mathbf{A}_L^T]^T$. Here, $\mathbf{\Psi}_l$ is the sub-dictionary and $\mathbf{A}_l$ is the coefficient matrix for level $l$. The approximation in level $l$ is expressed as
\begin{equation}
\mathbf{R}_{l-1} = \mathbf{\Psi}_l\mathbf{A}_l+\mathbf{R}_{l}, \text{ for } l = {1,...,L},
\label{eqn:rep_layer_l}
\end{equation}  where $\mathbf{R}_{l-1}$, $\mathbf{R}_{l}$ are the residuals for the levels $l-1$ and $l$ respectively and $\mathbf{R}_0 = \mathbf{Y}$, the matrix of training image patches. This implies that the residual matrix in level $l-1$ serves as the training data for level $l$. Note that the sparsity of the representation in each level is fixed at $1$. Hence, the overall approximation for all levels is
\begin{equation}
\mathbf{Y} = \sum_{l=1}^L \mathbf{\Psi}_l\mathbf{A}_l+\mathbf{R}_{L}.
\label{eqn:rep_layer_all}	
\end{equation}MLD learning can be interpreted as a block-based dictionary learning problem with unit sparsity per block, where the sub-dictionary in each block can allow only a $1$-sparse representation and each block corresponds to a level. The sub-dictionary for level $l$, $\mathbf{\Psi}_l$, is the set of cluster centroids learned from the training matrix for that level, $\mathbf{R}_{l-1}$, using K-hyperline clustering. MLD learning can be formally stated as an optimization problem that proceeds from the first level until the stopping criteria is reached. For level $l$, we solve
\begin{align}
\nonumber &\argmin_{\mathbf{\Psi}_l,\mathbf{A}_l} \|\mathbf{R}_{l-1} - \mathbf{\Psi}_l\mathbf{A}_l\|_F^2 \text{ subject to } \|\mathbf{a}_{l,i}\|_0 \leq 1, \\
& \phantom{\argmin_{\mathbf{\Psi}_l,\mathbf{A}_l} \|\mathbf{R}_{l-1} - \mathbf{\Psi}_l\mathbf{A}_l\|_F^2} \text{ for } i = \{1, . . ., T\},
\label{eqn:opt_multilev_khyp}
\end{align}along with the constraint that the columns of $\mathbf{\Psi}_l$ have unit $\ell_2$ norm,  where $\mathbf{a}_{l,i}$ is the $i^\text{th}$ column of $\mathbf{A}_l$ and $T$ is the number of columns in $\mathbf{A}_l$. We adopt the notation $\{\mathbf{\Psi}_l,\mathbf{A}_l\} = \text{KLC}(\mathbf{R}_{l-1},K_l)$ to denote the problem in (\ref{eqn:opt_multilev_khyp}) where $K_l$ is the number of atoms in $\mathbf{\Psi}_l$. The stopping criteria is provided either by imposing a limit on the residual representation error or the maximum number of levels ($L$). Note that the total number of levels is the same as the maximum number of non-zero coefficients (sparsity) of the representation. The error constraint can be stated as, $\|\mathbf{r}_{l,i}\|_2^2 \leq \epsilon, \forall i = {1, ..., T}$, where $\mathbf{r}_{l,i}$ is the $i^\text{th}$ column in $\mathbf{R}_l$, and $\epsilon$ is the error goal.

Table \ref{Table:Algorithm_MLD} lists the MLD learning algorithm with a fixed $L$. We use the notation $\Lambda_l(j)$ to denote the $j^\text{th}$ element of the set $\Lambda_l$. The set $\Lambda_l$ contains the indices of the residual vectors of level $l$ whose norm is greater than the error goal. The residual vectors indexed by $\Lambda_{l}$ are stacked in the matrix, $\hat{\mathbf{R}}_l$, which in turn serves as the training matrix for the next level, $l+1$. In MLD learning, for a given level $l$, the residual $\mathbf{r}_{l,i}$ is orthogonal to the representation $\mathbf{\Psi}_l \mathbf{a}_{l,i}$. This implies that
\begin{equation}
\|\mathbf{r}_{l-1,i}\|_2^2 = \|\mathbf{\Psi}_l \mathbf{a}_{l,i}\|_2^2+\|\mathbf{r}_{l,i}\|_2^2.
\label{eqn:rep_energy_vector}
\end{equation} Combining this with the fact that $\mathbf{y}_i = \sum_{l=1}^L \mathbf{\Psi}_l \mathbf{a}_{l,i} + \mathbf{r}_{L,i}$, $\mathbf{a}_{l,i}$ is $1-$sparse, and the columns of $\mathbf{\Psi}_l$ are of unit $\ell_2$ norm, we obtain the relation
\begin{equation}
\|\mathbf{y}_i\|_2^2 = \sum_{l=1}^L \|\mathbf{a}_{l,i}\|_2^2 + \|\mathbf{r}_{L,i}\|_2^2.
\label{eqn:rep_energy_vector1}
\end{equation} Equation (\ref{eqn:rep_energy_vector1}) states that the energy of any training vector is equal to the sum of squares of its coefficients and the energy of its residual. From (\ref{eqn:rep_energy_vector}), we also have that,
\begin{equation}
\|\mathbf{R}_{l-1}\|_F^2 = \|\mathbf{\Psi}_l \mathbf{A}_l\|_F^2+\|\mathbf{R}_l\|_F^2.
\label{eqn:rep_energy_relation}
\end{equation} The training vectors for the first level of the algorithm, $\mathbf{r}_{0,i}$ lie in the ambient $\mathbb{R}^M$ space and the residuals, $\mathbf{r}_{1,i}$, lie in a finite union of $\mathbb{R}^{M-1}$ subspaces. This is because, for each dictionary atom in the first level, its residual lies in an ${M-1}$ dimensional space orthogonal to it. In the second level, the dictionary atoms can possibly lie anywhere in $\mathbb{R}^{M}$, and hence the residuals can lie in a finite union of $\mathbb{R}^{M-1}$ and $\mathbb{R}^{M-2}$ dimensional subspaces. Hence, we can generalize that the dictionary atoms for all levels lie in $\mathbb{R}^{M}$, whereas the training vectors of level $l \geq 2$, lie in finite unions of $\mathbb{R}^{M-1}, \ldots, \mathbb{R}^{M-l+1}$ dimensional subspaces of the $\mathbb{R}^M$ space.

\subsection{Convergence}
\label{sec:convergence}
The convergence of MLD learning and the energy hierarchy in the representation obtained using an MLD can be shown by providing two guarantees. The first guarantee is that for a fixed number of atoms per level, the algorithm will converge to the required error within a sufficient number of levels. This is because the K-hyperline clustering makes the residual energy of the representation smaller than the energy of the training matrix at each level (i.e.) $\|\mathbf{R}_l\|_F^2 < \|\mathbf{R}_{l-1}\|_F^2$. This follows from (\ref{eqn:rep_energy_relation}) and the fact that $\|\mathbf{\Psi}_l \mathbf{A}_l\|_F^2 > 0$.

The second guarantee is that for a sufficient number of atoms per level, the representation energy in level $l$ will be less than the representation energy in level $l-1$. To show this, we first state that for a sufficient number of dictionary atoms per level, $\|\mathbf{\Psi}_l \mathbf{A}_l\|_F^2 > \|\mathbf{R}_l\|_F^2$. This means that for every $l$
\begin{equation}
\label{eqn:en_hier1}
\|\mathbf{R}_l\|_F^2 < \|\mathbf{\Psi}_l \mathbf{A}_l\|_F^2 < \|\mathbf{R}_{l-1}\|_F^2,
\end{equation} because of (\ref{eqn:rep_energy_relation}). This implies that $\|\mathbf{\Psi}_l \mathbf{A}_l\|_F^2 < \|\mathbf{\Psi}_{l-1} \mathbf{A}_{l-1}\|_F^2$, i.e., the energy of the representation in each level reduces progressively from $l=1$ to $l=L$, thereby exhibiting energy hierarchy.

\subsection{Estimating Number of Atoms in Each Level}
\label{sec:mdl_lev_est}
The number of atoms in each level of an MLD can be optimally estimated using an information theoretic criteria such as minimum description length (MDL) \cite{grunwald2005advances}. The broad idea is that the model order, which is the number of dictionary atoms here, is chosen to minimize the total description length needed for representing the model and the data given the model. The codelength for encoding the data $\mathbf{Y}$  given the model $\mathbf{\Theta}$ is given as the negative log likelihood $-\log p (\mathbf{Y}|\mathbf{\Theta})$. The description length for the model is the number of bits needed to code the model parameters.

In order to estimate the number of atoms in each level using the MDL principle, we need to make some assumptions on the residual obtained in each level. Our first assumption will be that the a fraction $\alpha$ of the total energy in each level $E_l$ will be represented at that level and the remaining energy $(1-\alpha)E_l$ will be the residual energy. The residual and the representation energy sum up to the total energy in each level because, the residual in any level of MLD is orthogonal to the representation in that level. Therefore, at any level $l$, the represented energy will be $\alpha (1-\alpha)^{l-1} E$ and the residual energy will be $(1-\alpha)^{l} E$, where $E$ is the total energy of training data at the first level. For simplicity, we also assume that the residual at each level follows the zero-mean multinormal distribution $\mathcal{N}(\mathbf{0},\sigma_l^2 \mathbf{I}_M)$. Combining these two assumptions, the variance is estimated as $\sigma_l^2 = \frac{1}{MT} (1-\alpha)^{l} E$.

The total MDL score, which is an indicator of the information-theoretic complexity, is the sum of the negative log likelihood and the number of bits needed to encode the model. Encoding the model includes encoding the non-zero coefficients, their location, and the dictionary elements themselves. The MDL score for level $l$ with the data $\mathbf{R}_{l-1}$, dictionary $\mathbf{\Psi}_l \in \mathbb{R}^{M \times K_l}$, and the coefficient matrix $\mathbf{A}_l$ is 
\begin{align}
\nonumber &\text{MDL}(\mathbf{R}_{l-1}|\mathbf{\Psi}_l,\mathbf{A}_{l},K_l) = \frac{1}{2\sigma_l^2} \sum_{i=1}^T \|\mathbf{r}_{l-1,i}-\mathbf{\Psi}_l \mathbf{a}_{l,i}\|_2^2\\
&+\frac{1}{2} T \log (MT)+T \log (T K_l) +\frac{1}{2} K_l M \log(M T).
\label{eqn:mdlscore}
\end{align} Here, the first term in the sum represents the data description length, which is also the negative log-likelihood of the data after ignoring the constant term. The second term is the number of bits needed to code the $T$ non-zero coefficients as reals where each coefficient is coded using $0.5 \log (MT)$ bits \cite{saito1994simultaneous}. The third term denotes the bits needed to code their locations which are integers between $1$ and $T K_l$, and the fourth term represents the total bits needed to code all the dictionary elements as reals. The optimal model order $K_l$ is the number of dictionary atoms that results in the least MDL score. In practice, we test a finite number of model orders and pick the one which results in the least score. As an example, we train a dictionary using $5000$ grayscale patches of size $8 \times 8$ from the BSDS dataset \cite{BSDS_dataset}. We preprocess the patches by vectorizing them and subtracting the mean of each vectorized patch from its elements. We perform MLD learning and estimate the estimate the optimal number of dictionary atoms in each level using $\alpha = 0.25$, for a maximum of $16$ levels. For the sub-dictionary in each level, the number of atoms were varied between $10$ and $50$, and one that provided the least MDL score was chosen as optimal. The first few levels and the last level of the MLD obtained using such procedure is shown in Figure \ref{Fig:dict_BSDS}. The minimum MDL score obtained in each level is shown in \ref{Fig:mdlscore}. From these two figures, clearly, the information-theoretic complexity of the sub-dictionaries increase with the number of levels, and the atoms themselves progress from being simple geometric structures to stochastic textures.


\begin{figure}[t]
\begin{minipage}[b]{1.0\linewidth}
 \centering
\includegraphics[width = 8.5cm]{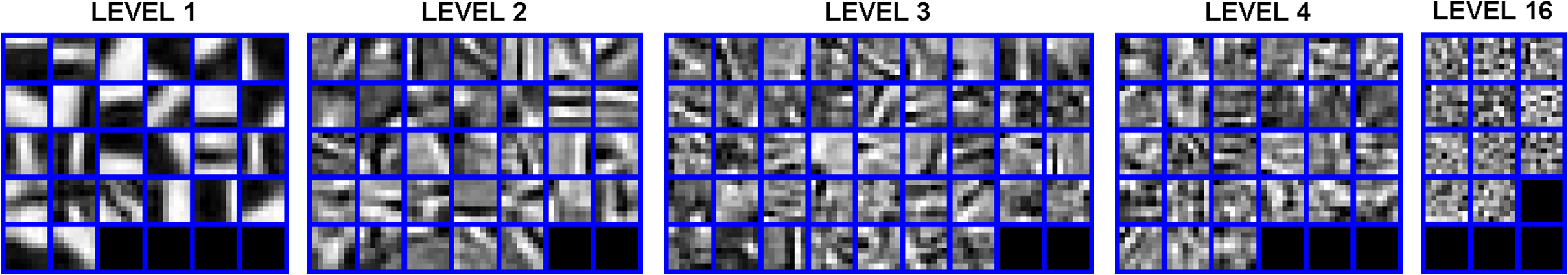}
\end{minipage}
\caption{The top $4$ levels and the last level of the MLD dictionary where the number of atoms are estimated using the MDL procedure. It comprises of geometric patterns in the first few levels and stochastic textures in the last level. Since each level has a different number of atoms, each sub-dictionary is padded with zero vectors, which appear as black patches.}
\label{Fig:dict_BSDS}
\end{figure}

\subsection{Sparse Approximation using an MLD}
\label{sec:sparse_approx_MLD}
In order to compute sparse codes for novel test data using a multilevel dictionary, we propose to perform reconstruction using a \textit{Multilevel Pursuit} (\textit{MulP}) procedure which evaluates a $1$-sparse representation for each level using the dictionary atoms from that level. Therefore, the coefficient vector for the $i^\text{th}$ data sample $\mathbf{r}_{l,i}$ in level $l$ is obtained using a simple \textit{correlate-and-max} operation, whereby we compute the correlations $\mathbf{\Psi}_l^T \mathbf{r}_{l,i}$ and pick the coefficient value and index corresponding to the maximum absolute correlation. The computational complexity of a \textit{correlate-and-max} operation is of order $M K_l$ and hence the complexity of obtaining the full representation using $L$ levels is of order $M K$, where $K = \sum_{i=1}^L K_l$ is the total number of atoms in the dictionary. Whereas, the complexity of obtaining an $L$ sparse representation on the full dictionary using Orthogonal Matching Pursuit is of order $L M K$.

\subsection{Robust Multilevel Dictionaries}
\label{sec:robust_MLD}
Although MLD learning is a simple procedure capable of handling large scale data with useful asymptotic generalization properties as described in Section (\ref{sec:gen_proof}), the procedure can be made robust and its generalization performance can be improved using randomization schemes. The Robust MLD (RMLD) learning scheme, which is closely related to \textit{Rvotes} \cite{breiman1999pasting} - a supervised ensemble learning method, improves the generalization performance of MLD as evidenced by Figure \ref{Fig:generalize}. The \textit{Rvotes} scheme randomly samples the training set to create $D$ sets of $T_D$ samples each, where $T_D \ll T$. The final prediction is obtained by averaging the predictions from the multiple hypotheses learned from the training sets. For learning level $l$ in RMLD, we draw $D$ subsets of randomly chosen training samples, $\{\mathbf{Y}_l^{(d)}\}_{d=1}^D$ from the original training set $\mathbf{Y}_l$ of size $T$, allowing for overlap across the subsets. Note that here, $\mathbf{Y}_l = \mathbf{R}_{l-1}$. The superscript here denotes the index of the subset. For each subset $\mathbf{Y}_l^{(d)}$ of size $T_D \ll T$, we learn a sub-dictionary $\mathbf{\Psi}_l^{(d)}$ with $K_l$ atoms using K-hyperline clustering. For each training sample in $\mathbf{Y}_l$, we compute $1-$sparse representations using all the $D$ sub-dictionaries, and denote the set of coefficient matrices as $\{\mathbf{A}_l^{(d)}\}_{d=1}^D$. The approximation for the $i^{\text{th}}$ training sample in level $l$, $\mathbf{y}_{l,i}$, is computed as the average of approximations using all $D$ sub-dictionaries, $\frac{1}{D} \sum_d \mathbf{\Psi}_l^{(d)} \mathbf{a}_{l,i}^{(d)}$. The ensemble approximations for all training samples in the level can be used to compute the set of residuals, and this process is repeated for a desired number of levels, to obtain an RMLD. 

Reconstruction of test data with an RMLD is performed by extending the multilevel pursuit. We obtain $D$ approximations for each data sample at a given level, average the approximations, compute the residual and repeat this for the subsequent levels. Note that this can also be implemented as multiple \textit{correlate-and-max} operations per data sample per level. Clearly, the computational complexity for obtaining a sparse representation using the RMLD is of order $D M K$, where $K = \sum_{i=1}^L K_l$.

\begin{figure}[t]
\begin{minipage}[b]{1.0\linewidth}
 \centering
\includegraphics[width = 8.5cm]{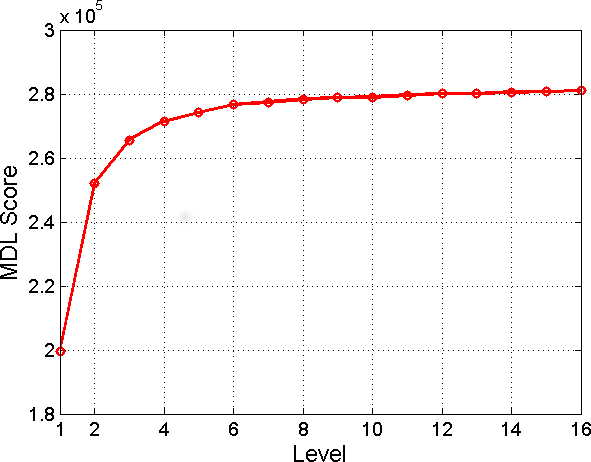}
\end{minipage}
\caption{The minimum MDL score of each level. The information-theoretic complexity of the sub-dictionaries increase with the number of levels.}
\label{Fig:mdlscore}
\end{figure}


\section{Stability and Generalization}
\label{sec:stability}
In this section, the behavior of the proposed dictionary learning algorithm is considered from the viewpoint of algorithmic stability: the behavior of the algorithm with respect to the perturbations in the training set. It will be shown that the dictionary atoms learned by the algorithm from two different training sets whose samples are realized from the same probability space, become arbitrarily close to each other, as the number of training samples $T \rightarrow \infty$. Since the proposed MLD learning is equivalent to learning K-hyperline cluster centroids in multiple levels, the stability analysis of K-hyperline clustering \cite{JT_khyp}, briefly discussed in Section \ref{sec:stability_intro}, will be utilized in order to prove its stability. For each level of learning, the cases of single and multiple minimizers to the clustering objective will be considered. Proving that the learning algorithm is stable will show that the global dictionaries learned from the data depend only on the probability space to which the training samples belong and not on the actual samples themselves, as $T \rightarrow \infty$. We also show that the MLD learning generalizes asymptotically, i.e., the difference between expected error and average empirical error in training approaches zero, as $T \rightarrow \infty$. Therefore, the expected error for novel test data, drawn from the same distribution as the training data, will approach the average empirical training error.

The stability analysis of the MLD algorithm will be performed  by considering two different dictionaries $\mathbf{\Psi}$ and $\mathbf{\Lambda}$ with $L$ levels each. Each level consists of $K_l$ dictionary atoms and the sub-dictionaries in each level are indicated by $\mathbf{\Psi}_l$ and $\mathbf{\Lambda}_l$ respectively. Sub-dictionaries $\mathbf{\Psi}_l$ and $\mathbf{\Lambda}_l$ are the cluster centers learned using K-hyperline clustering on the training data for level $l$. The steps involved in proving the overall stability of the algorithm are: (a) showing that each level of the algorithm is stable in terms of $L_1(P)$ distance between the distortion functions, defined in (\ref{eqn:dist_clustering}), as the number of training samples $T \rightarrow \infty$ (Section \ref{sec:level_stability}), (b) proving that stability in terms of $L_1(P)$ distances indicates closeness of the centers of the two clusterings (Section \ref{sec:dist_clust_cent}), in terms of the metric defined in (\ref{eqn:dist_clus_cent}), and (c) showing that level-wise stability leads to overall stability of the dictionary learning algorithm (Section \ref{sec:prob_space_res}).

\subsection{Level-wise Stability}
\label{sec:level_stability}
Let us define a probability space $(\mathcal{Y}_l,\mathbf{\Sigma}_l,P_l)$ where $\mathcal{Y}_l$ is the data that lies in $\mathbb{R}^M$, and $P_l$ is the probability measure. The training samples for the sub-dictionaries $\mathbf{\Psi}_l$ and $\mathbf{\Lambda}_l$ are two different sets of $T$ i.i.d. realizations from the probability space. We also assume that the $\ell_2$ norm of the training samples is bounded from above and below (i.e.), $0 < r \leq \|\mathbf{y}\|_2 \leq R < \infty$. Note that, in a general case, the data will lie in $\mathbb{R}^M$ for the first level of dictionary learning and in a finite union of lower-dimensional subspaces of $\mathbb{R}^M$ for the subsequent levels. In both cases, the following argument on stability will hold. This is because when the training data lies in a union of lower dimensional subspaces of $\mathbb{R}^M$, we can assume it to be still lying in $\mathbb{R}^M$, but assign the probabilities outside the union of subspaces to be zero.

The distortion function class for the clusterings, defined similar to (\ref{eqn:dist_func_class}), is uniform Donsker because the covering number with respect to the supremum norm grows polynomially, according to (\ref{eqn:cover_num}). When a unique minimizer exists for the clustering objective, the distortion functions corresponding to the different clusterings $\mathbf{\Psi}_l$ and $\mathbf{\Lambda}_l$ become arbitrarily close, $\|g_{\mathbf{\Psi}_l} - g_{\mathbf{\Lambda}_l}\|_{L_1(P_l)} \xrightarrow{P} 0$, even for completely disjoint training sets, as $T \rightarrow \infty$. However, in the case of multiple minimizers, $\|g_{\mathbf{\Psi}_l} - g_{\mathbf{\Lambda}_l}\|_{L_1(P_l)} \xrightarrow{P} 0$ holds only with respect to a change of $\displaystyle o(\sqrt{T})$ training samples between the two clusterings, and fails to hold for a change of $\Omega(\sqrt{T})$ samples \cite{kmeansRakhlin2007,JT_khyp}.

\subsection{Distance between Cluster Centers for a Stable Clustering}
\label{sec:dist_clust_cent}

For each cluster center in the clustering $\mathbf{\Psi}_l$, we pick the closest cluster center from $\mathbf{\Lambda}_l$, in terms of the distortion measure (\ref{eqn:clust_dist}), and form pairs. Let us indicate the $j^\mathrm{th}$ pair of cluster centers by $\boldsymbol{\psi}_{l,j}$ and $\boldsymbol{\lambda}_{l,j}$. Let us define $\tau$ disjoint sets $\{A_i\}_{i=1}^{\tau}$, in which the training data for the clusterings exist, such that $P_l(\cup_{i=1}^{\tau} A_i) = 1$. By defining such disjoint sets, we can formalize the notion of training data lying in a union of subspaces of $\mathbb{R}^M$. The intuitive fact that the cluster centers of two clusterings are close to each other, given that their distortion functions are close, is proved in the lemma below.

\begin{lemma}
\label{lem:dist_clust_cent}
Consider two sub-dictionaries (clusterings) ${\mathbf{\Psi}_l}$ and $\mathbf{\Lambda}_l$ with $K_l$ atoms each obtained using the $T$ training samples that exist in the $\tau$ disjoint sets $\{A_i\}_{i=1}^{\tau}$ in the $\mathbb{R}^M$ space, with $0 < r \leq \|\mathbf{y}\|_2 \leq R < \infty$, and $dP_l(\mathbf{y})/d\mathbf{y} > C$ in each of the sets. When the distortion functions become arbitrarily close to each other,
$\|g_{\mathbf{\Psi}_l} - g_{\mathbf{\Lambda}_l}\|_{L_1(P_l)} \xrightarrow{P} 0$ as $T \rightarrow \infty$, the smallest angle between the subspaces spanned by the cluster centers becomes arbitrarily close to zero, i.e.,
\begin{equation}
\label{eqn:dist_clust_cent_pair}
\angle(\boldsymbol{\psi}_{l,j},\boldsymbol{\lambda}_{l,j}) \xrightarrow{P} 0, \mathrm{ }, \forall \mathrm{ } j \in {1, \ldots, K_l}.
\end{equation}
\end{lemma}
\begin{proof}
Denote the smallest angle between the subspaces represented by $\boldsymbol{\psi}_{l,j}$ and $\mathbf{\lambda}_{l,j}$ as $\angle{(\boldsymbol{\psi}_{l,j},\mathbf{\lambda}_{l,j})} = \rho_{l,j}$ and define a region $S(\boldsymbol{\psi}_{l,j},\rho_{l,j}/2)=\{\mathbf{y}|\angle{(\boldsymbol{\psi}_{l,j},\mathbf{y})} \leq \rho_{l,j}/2, 0 < r \leq \|\mathbf{y}\|_2 \leq R < \infty\}$. If $\mathbf{y} \in S(\boldsymbol{\psi}_{l,j},\rho_{l,j}/2)$, then $\mathbf{y}^T (\mathbf{I}-\boldsymbol{\psi}_{l,j}\boldsymbol{\psi}_{l,j}^T) \mathbf{y} \leq \mathbf{y}^T (\mathbf{I}-\boldsymbol{\lambda}_{l,j}\boldsymbol{\lambda}_{l,j}^T) \mathbf{y}$. An illustration of this setup for a 2-D case is given in Figure \ref{Fig:MLD_stab1}. In this figure, the arc $\hat{\mathbf{q}}_1\hat{\mathbf{q}}_2$ is of radius $r$ and represents the minimum value of $\|\mathbf{y}\|_2$. By definition, the $L_1(P_l)$ distance between the distortion functions of the clusterings for data that exists in the disjoint sets $\{A_i\}_{i=1}^{\tau}$ is
\small
\begin{align}
& \|g_{\boldsymbol{\Psi}_l}-g_{\boldsymbol{\Lambda}_l}\|_{L_1(P_l)} = \sum_{i=1}^{\tau} \int_{A_i} |g_{\boldsymbol{\Psi}_l}(\mathbf{y}) - g_{\boldsymbol{\Lambda}_l}(\mathbf{y})| dP_{l}(\mathbf{y}).
\end{align} \normalsize For any $j$ and $i$ with a non-empty $B_{l,i,j} = S(\boldsymbol{\psi}_{l,j},\rho_{l,j}/2) \cap A_i$ we have,
\small
\begin{align}
\displaystyle{\|g_{\boldsymbol{\Psi}_l}-g_{\boldsymbol{\Lambda}_l}\|_{L_1(P_l)} \geq \int_{B_{l,i,j}} |g_{\boldsymbol{\Psi}_l}(\mathbf{y}) - g_{\boldsymbol{\Lambda}_l}(\mathbf{y})| dP_l (\mathbf{y})},
\end{align}
\begin{align}
\nonumber &= \int_{B_{l,i,j}} \big[\mathbf{y}^T \left(\mathbf{I}-\boldsymbol{\lambda}_{l,j}\boldsymbol{\lambda}_{l,j}^T\right) \mathbf{y}- \sum_{k=1}^K \mathbf{y}^T \left(\mathbf{I}-\boldsymbol{\psi}_{l,k}\boldsymbol{\psi}_{l,k}^T\right) \mathbf{y}
\label{eqn:dist_hyp_sph}\\
&\phantom{=\int_{B_{l,i,j}}} \mathbb{I} \left(\mathbf{y} \text{ closest to $\boldsymbol{\psi}_{l,k}$}\right) \big] dP_l(\mathbf{y}),
\end{align}
\begin{align}
& \geq \int_{B_{l,i,j}} \big[\mathbf{y}^T \left(\mathbf{I}-\boldsymbol{\lambda}_{l,j}\boldsymbol{\lambda}_{l,j}^T\right) \mathbf{y} - \mathbf{y}^T \left(\mathbf{I}-\boldsymbol{\psi}_{l,j}\boldsymbol{\psi}_{l,j}^T\right) \mathbf{y} \big] dP_l(\mathbf{y}),
\label{eqn:dist_hyp_sect}\\
& \geq C \int_{B_{l,i,j}} \left[\left(\mathbf{y}^T\boldsymbol{\psi}_{l,j}\right)^2-\left(\mathbf{y}^T\boldsymbol{\lambda}_{l,j}\right)^2\right] d\mathbf{y}.
\label{eqn:dist_hyp_sect1}
\end{align} \normalsize We have $g_{\boldsymbol{\Lambda}_l}(\mathbf{y}) = \mathbf{y}^T \left(\mathbf{I}-\boldsymbol{\lambda}_{l,j}\boldsymbol{\lambda}_{l,j}^T\right) \mathbf{y}$ in (\ref{eqn:dist_hyp_sph}), since $\boldsymbol{\lambda}_{l,j}$ is the closest cluster center to the data in $S(\boldsymbol{\psi}_{l,j},\rho_{l,j}/2) \cap A_i$ in terms of the distortion measure (\ref{eqn:clust_dist}). Note that $\mathbb{I}$ is the indicator function and (\ref{eqn:dist_hyp_sect1}) follows from (\ref{eqn:dist_hyp_sect}) because $dP_l(\mathbf{y})/d\mathbf{y} > C$. Since by assumption, $\|g_{\mathbf{\Psi}_l} - g_{\mathbf{\Lambda}_l}\|_{L_1(P_l)} \xrightarrow{P} 0$, from (\ref{eqn:dist_hyp_sect1}), we have
\begin{equation}
\left(\mathbf{y}^T\boldsymbol{\psi}_{l,j}\right)^2-\left(\mathbf{y}^T\boldsymbol{\lambda}_{l,j}\right)^2 \xrightarrow{P} 0,
\label{eqn:integ_zero}
\end{equation} because the integrand in (\ref{eqn:dist_hyp_sect1}) is a continuous non-negative function in the region of integration.

\begin{figure}[tb]
\begin{minipage}[b]{1.0\linewidth}
 \centering
\includegraphics[width = 6.5cm]{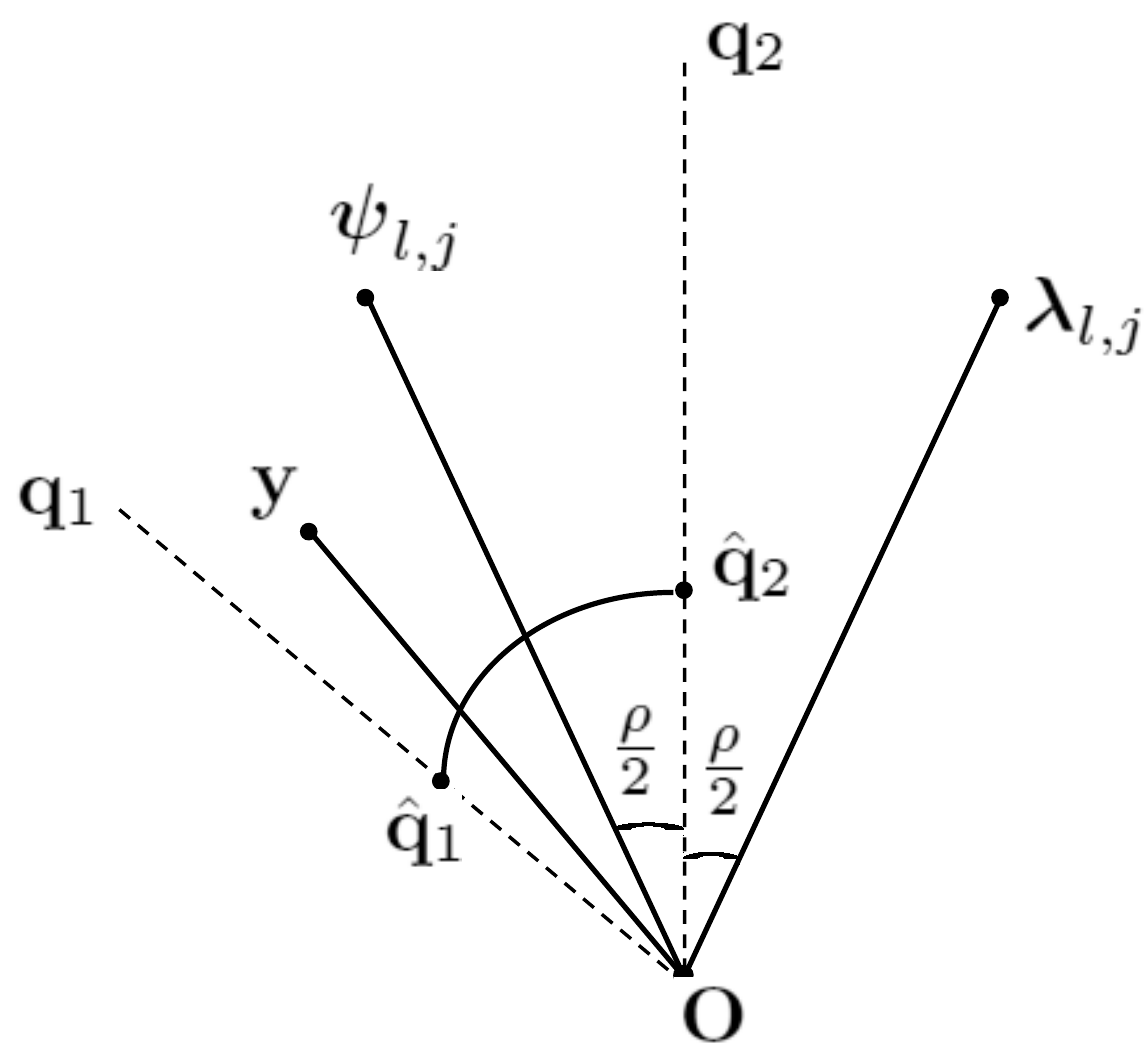}
\end{minipage}
\caption{Illustration for showing the stability of cluster centroids from the stability of distortion function.}
\label{Fig:MLD_stab1}
\end{figure}

Denoting the smallest angles between $\mathbf{y}$ and the subspaces spanned by $\boldsymbol{\psi}_{l,j}$ and $\boldsymbol{\lambda}_{l,j}$ to be $\theta_{\boldsymbol{\psi}_{l,j}}$ and $\theta_{\boldsymbol{\lambda}_{l,j}}$ respectively, from (\ref{eqn:integ_zero}), we have $\|\mathbf{y}\|_2^2(\cos^2 \theta_{\boldsymbol{\psi}_{l,j}} - \cos^2 \theta_{\boldsymbol{\lambda}_{l,j}}) \xrightarrow{P} 0$, for all $\mathbf{y}$. By definition of the region $B_{l,i,j}$, we have $\theta_{\boldsymbol{\psi}_{l,j}} \leq \theta_{\boldsymbol{\lambda}_{l,j}}$. Since $\|\mathbf{y}\|_2$ is bounded away from zero and infinity, if $(\cos^2 \theta_{\boldsymbol{\psi}_{l,j}} - \cos^2 \theta_{\boldsymbol{\lambda}_{l,j}}) \xrightarrow{P} 0$ holds for all $\mathbf{y} \in B_{l,i,j}$, then we have $\angle(\boldsymbol{\psi}_{l,j},\boldsymbol{\lambda}_{l,j}) \xrightarrow{P} 0$. This is true for all cluster center pairs as we have shown this for an arbitrary $i$ and $j$.
\end{proof}

\subsection{Stability of the MLD Algorithm}
\label{sec:prob_space_res}
The stability of the MLD algorithm as a whole, is proved in Theorem \ref{thm:stab_dict_learn} from its level-wise stability by using an induction argument. The proof will depend on the following lemma which shows that the residuals from two stable clusterings belong to the same probability space.

\begin{lemma}
\label{lem:prob_space_residuals}
When the training vectors for the sub-dictionaries (clusterings) $\mathbf{\Psi}_l$ and $\mathbf{\Lambda}_l$ are obtained from the probability space $(\mathcal{Y}_l,\mathbf{\Sigma}_l,P_l)$, and the cluster center pairs become arbitrarily close to each other as $T \rightarrow \infty$, the residual vectors from both the clusterings belong to an identical probability space $(\mathcal{Y}_{l+1},\mathbf{\Sigma}_{l+1},P_{l+1})$.
\end{lemma}

\begin{proof}
For the $j^\mathrm{th}$ cluster center pair $\boldsymbol{\psi}_{l,j}$, $\boldsymbol{\lambda}_{l,j}$, define $\bar{\mathbf{\Psi}}_{l,j}$ and $\bar{\mathbf{\Lambda}}_{l,j}$ as the projection matrices for their respective orthogonal complement subspaces $\boldsymbol{\psi}_{l,j}^{\perp}$ and $\boldsymbol{\lambda}_{l,j}^{\perp}$. Define the sets $D_{\boldsymbol{\psi}_{l,j}} = \{\mathbf{y} \in \bar{\mathbf{\Psi}}_{l,j}(\boldsymbol{\beta}+d\boldsymbol{\beta})+\boldsymbol{\psi}_{l,j}\alpha\}$ and $D_{\boldsymbol{\lambda}_{l,j}} = \{ \mathbf{y} \in \bar{\mathbf{\Lambda}}_{l,j}(\boldsymbol{\beta}+d\boldsymbol{\beta})+\boldsymbol{\lambda}_{l,j}\alpha\}$, where $-\infty < \alpha < \infty$, $\boldsymbol{\beta}$ is an arbitrary fixed vector, not orthogonal to both $\boldsymbol{\psi}_{l,j}$ and $\boldsymbol{\lambda}_{l,j}$, and $d\boldsymbol{\beta}$ is a differential element. The residual vector set for the cluster $\boldsymbol{\psi}_{l,j}$, when $\mathbf{y} \in D_{\boldsymbol{\psi}_{l,j}}$ is given by, $\mathbf{r}_{\boldsymbol{\psi}_{l,j}} \in \{\bar{\mathbf{\Psi}}_{l,j} \mathbf{y} | \mathbf{y} \in D_{\boldsymbol{\psi}_{l,j}}\}$, or equivalently $\mathbf{r}_{\boldsymbol{\psi}_{l,j}} \in \{\bar{\mathbf{\Psi}}_{l,j} (\boldsymbol{\beta}+d\boldsymbol{\beta})\}$. Similarly for the cluster $\boldsymbol{\lambda}_{l,j}$, we have $\mathbf{r}_{\boldsymbol{\lambda}_{l,j}} \in \{\bar{\mathbf{\Lambda}}_{l,j} (\boldsymbol{\beta}+d\boldsymbol{\beta})\}$. For a 2-D case, Figure \ref{Fig:MLD_stab2} shows the 1-D subspace $\boldsymbol{\psi}_{l,j}$, its orthogonal complement $\boldsymbol{\psi}_{l,j}^{\perp}$, the set $D_{\boldsymbol{\psi}_{l,j}}$ and the residual set $\{\bar{\mathbf{\Psi}}_{l,j} (\boldsymbol{\beta}+d\boldsymbol{\beta})\}$.

In terms of probabilities, we also have that $P_l(\mathbf{y} \in D_{\boldsymbol{\psi}_{l,j}}) = P_{l+1}(\mathbf{r}_{\boldsymbol{\psi}_{l,j}} \in \{\bar{\mathbf{\Psi}}_{l,j} (\boldsymbol{\beta}+d\boldsymbol{\beta})\})$, because the residual set $\{\bar{\mathbf{\Psi}}_{l,j} (\boldsymbol{\beta}+d\boldsymbol{\beta})\}$ is obtained by a linear transformation of $D_{\boldsymbol{\psi}_{l,j}}$. Here $P_l$ and $P_{l+1}$ are probability measures defined on the training data for levels $l$ and $l+1$ respectively. Similarly, $P_l(\mathbf{y} \in D_{\boldsymbol{\lambda}_{l,j}}) = P_{l+1}(\mathbf{r}_{\boldsymbol{\lambda}_{l,j}} \in \{\bar{\mathbf{\Lambda}}_{l,j} (\boldsymbol{\beta}+d\boldsymbol{\beta})\})$. When $T \rightarrow \infty$, the cluster center pairs become arbitrarily close to each other, i.e., $\angle(\boldsymbol{\psi}_{l,j}, \boldsymbol{\lambda}_{l,j}) \xrightarrow{P} 0$, by assumption. Therefore, the symmetric difference between the sets $D_{\boldsymbol{\psi}_{l,j}}$ and $D_{\boldsymbol{\lambda}_{l,j}}$ approaches the null set, which implies that $P_l(\mathbf{y} \in D_{\boldsymbol{\psi}_{l,j}})-P_l(\mathbf{y} \in D_{\boldsymbol{\lambda}_{l,j}}) \rightarrow 0$. This implies,
\begin{align}
\label{eqn:prob_space_res}
\nonumber P_{l+1}(\mathbf{r}_{\boldsymbol{\psi}_{l,j}} & \in \{\bar{\mathbf{\Psi}}_{l,j} (\boldsymbol{\beta}+d\boldsymbol{\beta})\})-\\
& P_{l+1}(\mathbf{r}_{\boldsymbol{\lambda}_{l,j}} \in \{\bar{\mathbf{\Lambda}}_{l,j} (\boldsymbol{\beta}+d\boldsymbol{\beta})\}) \rightarrow 0,
\end{align} for an arbitrary $\boldsymbol{\beta}$ and $d\boldsymbol{\beta}$, as $T \rightarrow \infty$. This means that the residuals of $\boldsymbol{\psi}_{l,j}$ and $\boldsymbol{\lambda}_{l,j}$ belong to a unique but identical probability space. Since we proved this for an arbitrary $l$ and $j$, we can say that the residuals of clusterings $\mathbf{\Psi}_l$ and $\mathbf{\Lambda}_l$ belong to an identical probability space given by $(\mathcal{Y}_{l+1},\mathbf{\Sigma}_{l+1},P_{l+1})$.
\end{proof}

\begin{figure}[tb]
\begin{minipage}[b]{1.0\linewidth}
 \centering
 \centerline{\includegraphics[width = 7.5cm]{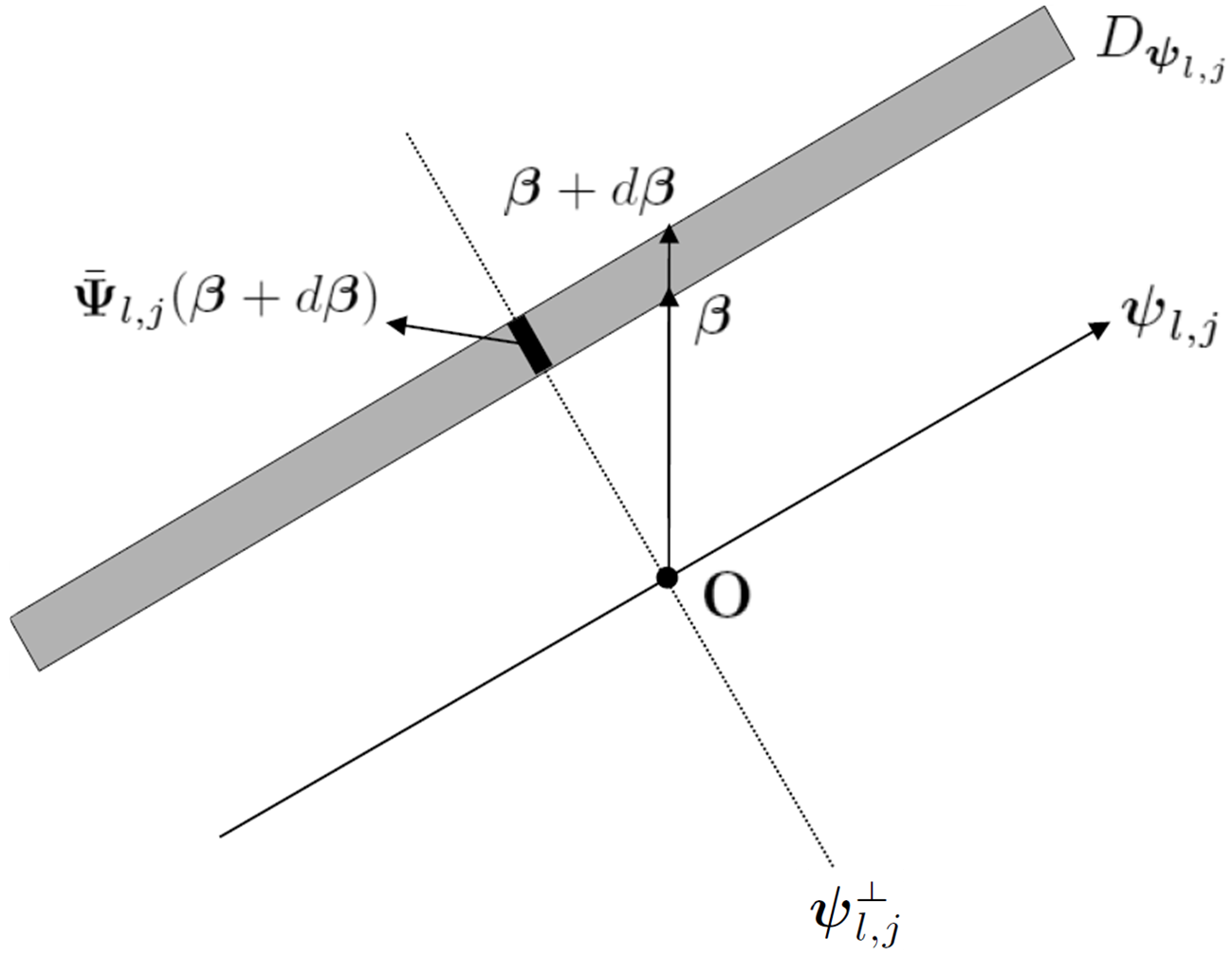}}
\end{minipage}
\caption{The residual set $\{\bar{\mathbf{\Psi}}_{l,j} (\boldsymbol{\beta}+d\boldsymbol{\beta})\}$, for the 1-D subspace $\boldsymbol{\psi}_{l,j}$, lying in its orthogonal complement subspace $\boldsymbol{\psi}_{l,j}^{\perp}$.}
\label{Fig:MLD_stab2}
\end{figure}

\begin{theorem}
\label{thm:stab_dict_learn}
Given that the training vectors for the first level are generated from the probability space $(\mathcal{Y}_1,\mathbf{\Sigma}_1,P_1)$, and the norms of training vectors for each level are bounded as $0 < r \leq \|\mathbf{y}\|_2 \leq R < \infty$, the MLD learning algorithm is stable as a whole.
\end{theorem}

\begin{proof}
The level-wise stability of MLD was shown in Section \ref{sec:level_stability}, for two cases: (a) when a unique minimizer exists for the distortion function and (b) when a unique minimizer does not exist. Lemma \ref{lem:dist_clust_cent} proved that the stability in terms of closeness of distortion functions implied stability in terms of learned cluster centers. For showing the level-wise stability, we assumed that the training vectors in level $l$ for clusterings $\mathbf{\Psi}_l$ and $\mathbf{\Lambda}_l$ belonged to the same probability space. However, when learning the dictionary, this is true only for the first level, as we supply the algorithm with training vectors from the probability space $(\mathcal{Y}_1,\mathbf{\Sigma}_1,P_1)$.

We note that the training vectors for level $l+1$ are residuals of the clusterings $\mathbf{\Psi}_l$ and $\mathbf{\Lambda}_l$. Lemma \ref{lem:prob_space_residuals} showed that the residuals of level $l$ for both the clusterings belong to an identical probability space $(\mathcal{Y}_{l+1},\mathbf{\Sigma}_{l+1},P_{l+1})$, given that the training vectors of level $l$ are realizations from the probability space $(\mathcal{Y}_l,\mathbf{\Sigma}_l,P_l)$ and $T \rightarrow \infty$. By induction, this along with the fact that the training vectors for level $1$ belong to the same probability space $(\mathcal{Y}_1,\mathbf{\Sigma}_1,P_1)$, shows that all the training vectors of both the dictionaries for any level $l$ indeed belong to a probability space $(\mathcal{Y}_{l},\mathbf{\Sigma}_{l},P_{l})$ corresponding to that level. Hence all the levels of the dictionary learning are stable and the MLD learning is stable as a whole. Similar to K-hyperline clustering, if there are multiple minimizers in at least one level, the algorithm is stable only with respect to a change of $o(\sqrt{T})$ training samples between the two clusterings and failts to hold for a change of $\Omega(\sqrt{T})$ samples.
\end{proof}


\begin{figure}[t]
\begin{minipage}[c]{1.0\linewidth}
 \centering
 \includegraphics[width = 8.5cm]{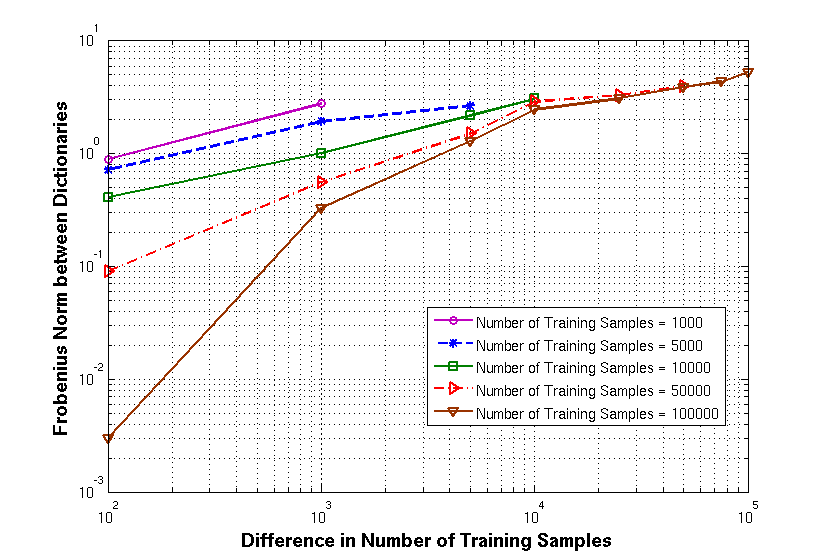}
\end{minipage}
\caption{Demonstration of the stability behavior of the proposed MLD learning algorithm. The minimum Frobenius norm between difference of two dictionaries with respect to permutation of their columns and signs is shown. The second dictionary is obtained by replacing different number of samples in the training set, used for training the original dictionary, with new data samples.}
\label{Fig:stability}
\end{figure}

\subsection{Generalization Analysis}
\label{sec:gen_proof}
Since our learning algorithm consists of multiple levels, and cannot be expressed as an ERM on a whole, the algorithm can be said to generalize asymptotically if the sum of empirical errors for all levels converge to the sum of expected errors, as $T \rightarrow \infty$. This can be expressed as 
\begin{equation}
\label{eqn:gen_eqn}
\left| \frac{1}{T} \sum_{l=1}^L \sum_{i=1}^T g_{\mathbf{\Psi}_l}(\mathbf{y}_{l,i}) - \sum_{l=1}^L \mathbb{E}_{P_l}[g_{\mathbf{\Psi}_l}]\right| \xrightarrow{P} 0,
\end{equation} where the training samples for level $l$ given by $\{\mathbf{y}_{l,i}\}_{i=1}^T$ are obtained from the probability space $(\mathcal{Y}_l,\mathbf{\Sigma}_l,P_l)$. When (\ref{eqn:gen_eqn}) holds and the learning algorithm generalizes, it can be seen that the expected error for test data which is drawn from the same probability space as that of the training data, is close to the average empirical error. Therefore, when the cluster centers for each level are obtained by minimizing the empirical error, the expected test error will also be small. 

In order to show that (\ref{eqn:gen_eqn}) holds, we use the fact that each level of MLD learning is obtained using K-hyperline clustering. Hence, from (\ref{eqn:donsker}), the average empirical distortion in each level converges to the expected distortion as $T\rightarrow\infty$, 
\begin{equation}
\label{eqn:gen_eqn1}
\left| \frac{1}{T} \sum_{i=1}^T g_{\mathbf{\Psi}_l}(\mathbf{y}_{l,i}) - \mathbb{E}_{P_l}[g_{\mathbf{\Psi}_l}]\right| \xrightarrow{P} 0.
\end{equation} The validity of the condition in (\ref{eqn:gen_eqn}) follows directly from the triangle inequality,
\begin{align}
\nonumber &\left| \frac{1}{T} \sum_{l=1}^L \sum_{i=1}^T g_{\mathbf{\Psi}_l}(\mathbf{y}_{l,i}) - \sum_{l=1}^L \mathbb{E}_{P_l}[g_{\mathbf{\Psi}_l}]\right| \\
&\leq \sum_{l=1}^L \left| \frac{1}{T}  \sum_{i=1}^T g_{\mathbf{\Psi}_l}(\mathbf{y}_{l,i}) - \mathbb{E}_{P_l}[g_{\mathbf{\Psi}_l}]\right|.
\label{eqn:gen_eqn2}
\end{align} 

\begin{figure}
\begin{minipage}[c]{1.0\linewidth}
 \centering
 \includegraphics[width =7cm]{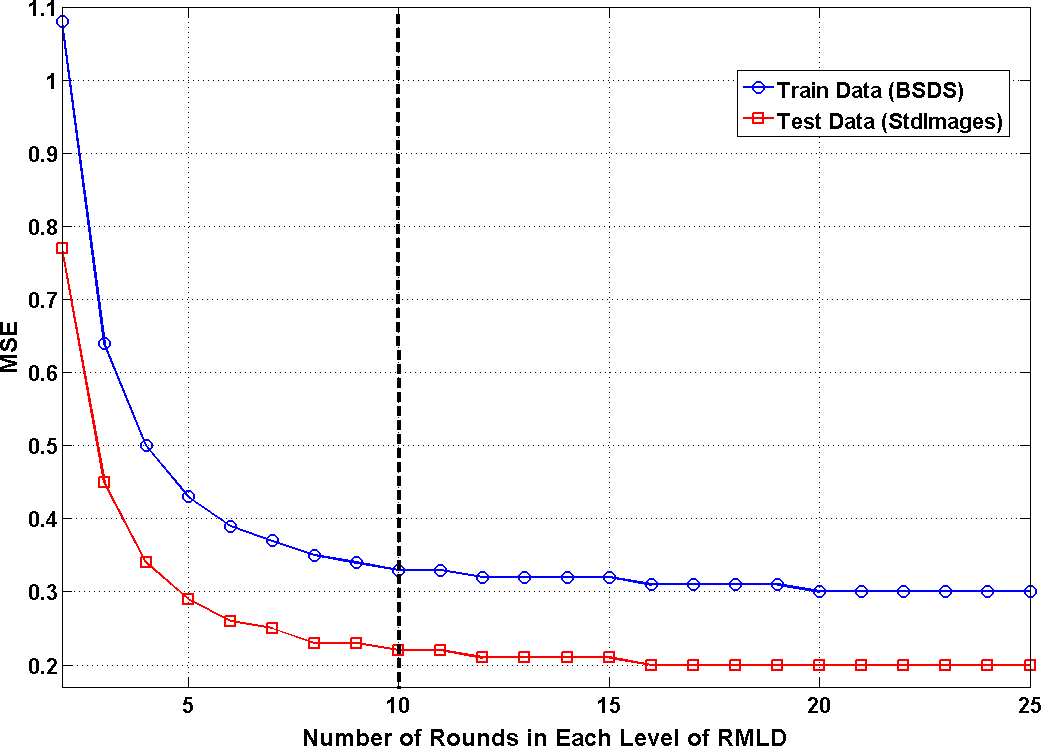}
\end{minipage}
\caption{Choosing the number of rounds ($R$) in RMLD learning. In this demonstration, RMLD design was carried out using $100,000$ samples and we observed that beyond $10$, both the train MSE and the test MSE do not change significantly.}
\label{Fig:rmld_rounds}
\end{figure}

If the \textit{MulP} coding scheme is used for test data, and the training and test data for level $1$ are obtained from the probability space $(\mathcal{Y}_1,\mathbf{\Sigma}_1,P_1)$, the probability space for both training and test data in level $l$ will be $(\mathcal{Y}_l,\mathbf{\Sigma}_l,P_l)$. This is because, both the \textit{MulP} coding scheme and MLD learning associate the data to a dictionary atom using the maximum absolute correlation measure and create a residual that is orthogonal to the atom chosen in a level. Hence, the assumption that training and test data are drawn from the same probability space in all levels hold and the expected test error will be similar to the average empirical training error.

\section{Simulation Results}
\label{sec:exp}
In this section, we present experiments to demonstrate the stability and generalization characteristics of a multilevel dictionary, and evaluate its use in compressed recovery of images and subspace learning. Both stability and generalization are crucial for building effective global dictionaries that can model patterns in any novel test image. Although it is not possible to demonstrate the asymptotic behavior experimentally, we study the changes in the behavior of the learning algorithm with increase in the number of samples used for training. Compressed recovery is a highly relevant application for global dictionaries, since it is not possible to infer dictionaries with good reconstructive power directly from the low-dimensional random measurements of image patches. It is typical to employ both $\ell_1$ minimization and greedy pursuit methods for recovering images from their compressed measurements. Though $\ell_1$ minimization incurs higher computational complexity, it often provides improved recovery performance when compared to greedy approaches. Hence, it is important to compare its recovery performance to that of the MLD that uses a simple greedy pursuit. Subspace learning is another application that can benefit from the use of multilevel dictionaries. In subspace learning, it is common to obtain a linear embedding from the training data, and apply it to novel test data for dimensionality reduction, classification, and visualization. These approaches can be unsupervised (eg. Principal Component Analysis, Locality Preserving Projections) or can use the class label information while learning the embedding (eg. Linear Discriminant Analysis, Local Discriminant Embedding). Several subspace learning algorithms can be unified under the framework of graph embedding \cite{yan2007graph}, wherein an undirected graph describing the relation between the data samples is provided as the input. We propose to use graphs constructed based on sparse codes, from a multilevel dictionary, for subspace learning in both supervised and unsupervised settings.

\begin{figure}
\begin{minipage}[c]{1.0\linewidth}
 \centering
 \includegraphics[width =7cm]{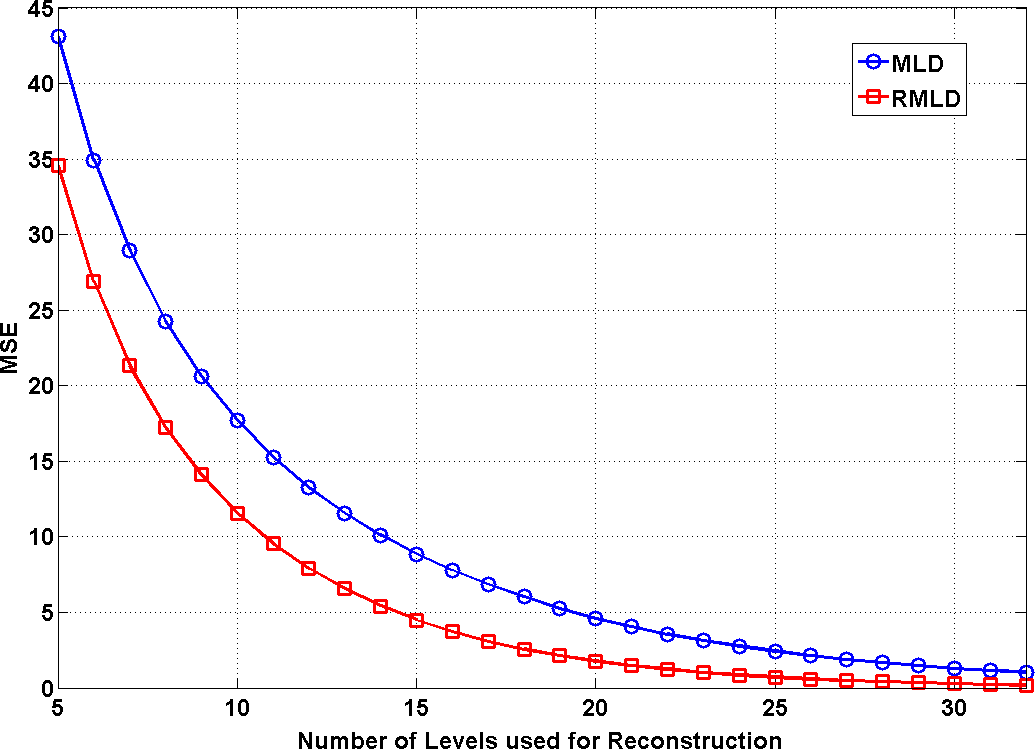}
\end{minipage}
\caption{Reconstruction of novel test data using MLD and RMLD dictionaries for the case $T = 100,000$. The approximation error is plotted against the number of levels used for reconstruction with both the dictionaries.}
\label{Fig:levmse}
\end{figure}

All simulations for stability/generalization, and compressed recovery use dictionaries trained on image patches from the Berkeley Segmentation Dataset (BSDS) \cite{BSDS_dataset}. The BSDS dataset contains a total of $400$ images and the number of patches used in our experiments vary between $5000$ and $400,000$. The images were converted to grayscale and no other preprocessing was performed on these images. We used patches of size $8 \times 8$ and no noise was added to the patches. For evaluating the performance of the dictionaries, we considered $8$ standard images (\textit{Barbara, Boat, House, Lena, Couple, Fingerprint, Man, Peppers}). For the subspace learning simulations, we used the Forest Covertype dataset \cite{forest} which consists of $581,012$ samples belonging to $7$ different classes. As per the standard procedure, we used the first $15,120$ samples ($2160$ per class) for training and the rest for testing.

\subsection{Stability}
\label{sec:stab}
In order to illustrate the stability characteristics of MLD learning, we setup an experiment where we consider a multilevel dictionary of $4$ levels, with $8$ atoms in each level. We trained multilevel dictionaries using different number of training patches $T$.  As we showed in Section \ref{sec:stability}, asymptotic stability is guaranteed when the training set is changed by not more than $o(\sqrt{T})$ samples. The inferred dictionary atoms will not vary significantly, if this condition is satisfied. We fixed the size of the training set at different values $T = \{$1000, 5000, 10,000, 50,000, 100,000$\}$ and learned an initial set of dictionaries using the proposed algorithm. The second set of dictionaries were obtained by replacing different number of samples from the original training set. For each case of $T$, the number of replaced samples was varied between $100$ and $T$. For example, when $T = 10,000$, the number of replaced training samples were $100, 1000, 5000,$ and $10,000$. The amount of change between the initial and the second set of dictionaries was quantified using the minimum Frobenius norm of their difference with respect to permutations of their columns and sign changes. In Figure \ref{Fig:stability}, we plot this quantity for different values of $T$ as a function of the number of samples replaced in the training set. For each case of $T$, the difference between the dictionaries increases as we increase the replaced number of training samples. Furthermore, for a fixed number of replaced samples (say $100$), the difference reduces with the increase in the number of training samples, since it becomes closer to asymptotic behavior.

\begin{figure}
\begin{minipage}[c]{1.0\linewidth}
 \centering
 \includegraphics[width =7.5cm]{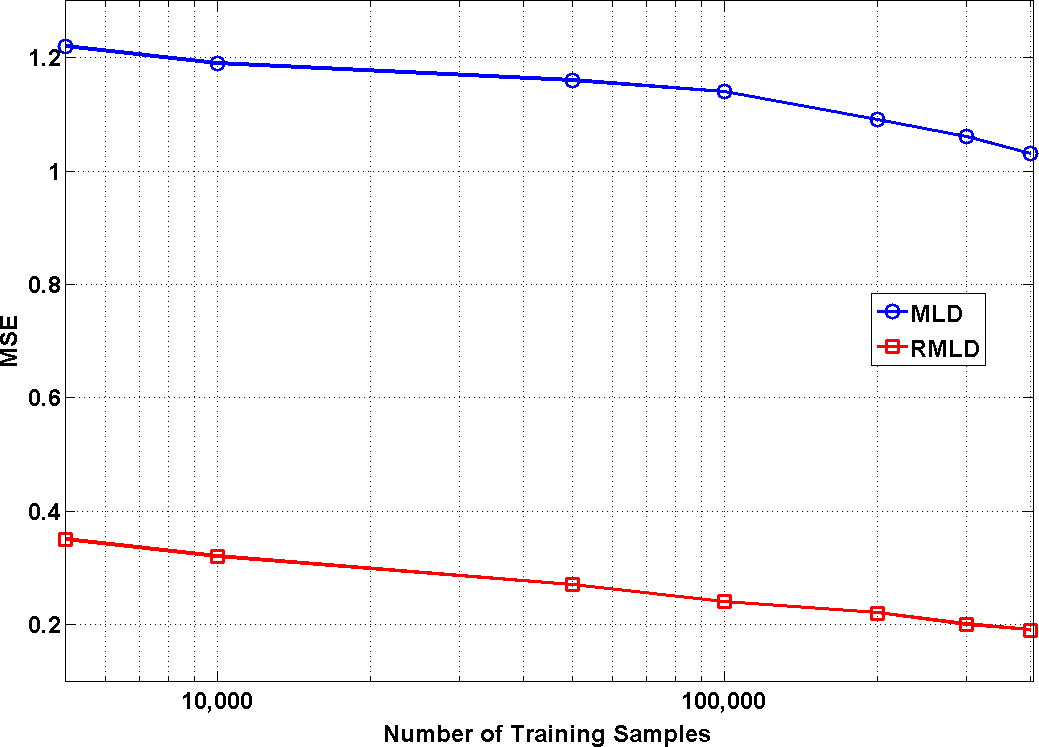}
\end{minipage}
\caption{Demonstration of the generalization characteristics of the proposed MLD and RMLD algorithms. We plot the MSE obtained by representing patches from the test dataset, using dictionaries learned with different number of training patches.}
\label{Fig:generalize}
\end{figure}

\subsection{Generalization}
\label{sec:gen}
Generalization of a dictionary learning algorithm guarantees a small approximation error for a test data sample, if the training samples are well approximated by the dictionary. In order to demonstrate the generalization characteristics of MLD learning, we designed dictionaries using different number of training image patches, of size $8 \times 8$, and evaluated the sparse approximation error for patches in the test dataset. The test dataset consisted of $120,000$ patches chosen randomly from the $8$ standard images. For multilevel learning, we fixed the number of levels at $32$, and used the approach proposed in Section \ref{sec:mdl_lev_est} to estimate the number of atoms needed in each level ($\alpha = 0.5$). Similarly, we fixed the number of levels at $32$ for the RMLD learning. Since RMLD learning does not require careful choice of the number of atoms in each level, we fixed $K_\ell = 32$. Though learning multiple sets of atoms in each level can lead to improved generalization, the benefit seems to level off after a certain number of rounds. As an example, let us consider the case where $T = 100,000$ and vary the number of rounds in RMLD between $2$ and $25$. As described in Section \ref{sec:robust_MLD}, increasing the number of rounds results in higher computational complexity while evaluating the sparse codes. Figure \ref{Fig:rmld_rounds} illustrates the MSE on the training data and the test data obtained using RMLD with different number of rounds in each level. Since the MSE did not vary significantly beyond $10$ rounds, we fixed $R = 10$ in our reconstruction experiments. 

\begin{table*}[tb]
\renewcommand{\arraystretch}{1.3}
  \centering
\caption{PSNR (dB) of the images recovered from compressed measurements obtained using Gaussian random measurement matrices. Results obtained with the Online (OMP), Online ($\ell_1$), RMLD (MulP), and MLD (MulP) algorithms are given in clockwise order beginning from top left corner. Higher PSNR for each case is indicated in bold font.}
\begin{tabular}{|c||c||c|c||c|c||c|c||c|c||c|c|}
\hline
\textbf{SNR (dB)}&\textbf{\# Measurements}&\multicolumn{2}{c||}{\textbf{Boat}}&\multicolumn{2}{c||}{\textbf{House}}&\multicolumn{2}{c||}{\textbf{Lena}}&\multicolumn{2}{c||}{\textbf{Man}}&\multicolumn{2}{c|}{\textbf{Peppers}}\\
\hline
\hline
\multirow{6}{*}{0}&\multirow{2}{*}{8}&21.43&22.32&22.43&23.36&23.13&24.05&22.23&23.13&19.08&20.01\\
\cline{3-12}
&&21.60&\textbf{22.43}&22.86&\textbf{23.76}&23.52&\textbf{24.39}&22.43&\textbf{23.28}&19.55&\textbf{20.43}\\
\hhline{|~||===========}
&\multirow{2}{*}{16}&22.19&23.19&23.31&24.39&24.02&25.03&22.98&23.97&19.97&21.05\\
\cline{3-12}
&&22.67&\textbf{23.60}&24.18&\textbf{25.15}&24.76&\textbf{25.75}&23.51&\textbf{24.46}&20.96&\textbf{21.95}\\
\hhline{|~||===========}
&\multirow{2}{*}{32}&23.50&24.18&24.94&25.48&25.54&26.08&24.26&24.95&21.71&22.14\\
\cline{3-12}
&&24.18&\textbf{25.15}&25.94&\textbf{27.03}&26.46&\textbf{27.54}&25.02&\textbf{26.02}&22.87&\textbf{23.90}\\
\hline

\hline
\multirow{6}{*}{15}&\multirow{2}{*}{8}&22.77&23.65&23.95&24.97&24.61&25.60&23.58&24.46&20.61&21.68\\
\cline{3-12}
&&23.48&\textbf{24.45}&25.11&\textbf{26.14}&25.69&\textbf{26.70}&24.34&\textbf{25.31}&21.90&\textbf{22.90}\\
\hhline{|~||===========}
&\multirow{2}{*}{16}&23.94&26.33&25.36&28.65&26.03&28.92&24.74&27.08&22.28&25.43\\
\cline{3-12}
&&25.29&\textbf{26.56}&27.43&\textbf{28.71}&27.83&\textbf{29.12}&26.09&\textbf{27.36}&24.34&\textbf{25.60}\\
\hhline{|~||===========}
&\multirow{2}{*}{32}&26.33&\textbf{30.19}&28.16&\textbf{33.59}&28.55&\textbf{33.17}&26.48&\textbf{30.64}&25.21&\textbf{29.78}\\
\cline{3-12}
&&28.13&29.96&30.77&33.41&30.88&32.94&28.81&30.44&27.48&29.47\\

\hline

\hline
\multirow{6}{*}{25}&\multirow{2}{*}{8}&22.82&23.73&24.01&25.09&24.66&25.70&23.63&24.54&20.67&21.83\\
\cline{3-12}
&&23.62&\textbf{24.56}&25.27&\textbf{26.30}&25.85&\textbf{26.83}&24.47&\textbf{25.42}&22.05&\textbf{23.04}\\
\hhline{|~||===========}
&\multirow{2}{*}{16}&24.00&26.57&25.44&29.11&26.10&29.30&24.81&27.32&22.37&25.87\\
\cline{3-12}
&&25.55&\textbf{26.84}&27.80&\textbf{29.23}&28.15&\textbf{29.48}&26.35&\textbf{27.63}&24.68&\textbf{25.99}\\
\hhline{|~||===========}
&\multirow{2}{*}{32}&26.38&\textbf{30.77}&28.71&\textbf{34.81}&28.61&\textbf{33.98}&27.13&\textbf{31.15}&25.97&\textbf{30.74}\\
\cline{3-12}
&&28.72&30.45&31.67&34.57&31.63&33.67&29.37&30.87&28.28&30.54\\
\hline
\end{tabular}
\label{table:cs_perf1}
\end{table*}

Figure \ref{Fig:levmse} compares the approximation error (MSE) obtained for the test dataset with MLD and RMLD ($10$ rounds) respectively, for the case of $T = 100,000$. The figure plots the MSE against the number of levels used in the reconstruction algorithm. Figure \ref{Fig:generalize} shows the approximation error (MSE) for the test image patches obtained with MLD and RMLD dictionaries learned using different number of training samples (varied between $5000$ and $400,000$). Since we proved in Section \ref{sec:gen_proof} the MLD learning generalizes asymptotically, we expect the approximation error for the test data to reduce with the increase in the size of the training set. From both these figures, it is clear that the RMLD scheme results in improved approximation of novel test patches when compared to MLD.

\subsection{Application: Compressed Recovery}
\label{sec:comp_rec}
In compressed recovery, an image is recovered using the low-dimensional random projections obtained from its patches. The performance of compressed recovery based on random measurement systems is compared for MLD, RMLD and online dictionaries. For the case of online dictionaries learned using the algorithm described in \cite{Mairal_Bach_Ponce_Sapiro_2009}, we report results obtained using both $\ell_1$ minimization and the OMP algorithm.  Sensing and recovery were performed on a patch-by-patch basis, on non-overlapping patches of size $8 \times 8$. The multilevel dictionaries were obtained with the parameters described in the previous section, using $400,000$ training samples. The online dictionary was trained using the same training set, with the number of atoms fixed at $1024$. The measurement process can described as $\mathbf{x} = \mathbf{\Phi} \mathbf{\Psi} \mathbf{a} + \boldsymbol{\eta}$ where $\mathbf{\Psi}$ is the dictionary, $\mathbf{\Phi}$ is the measurement or projection matrix, $\boldsymbol{\eta}$ is the AWGN vector added to the measurement process, $\mathbf{x}$ is the output of the measurement process, and $\mathbf{a}$ is the sparse coefficient vector such that $\mathbf{y} = \mathbf{\Psi} \mathbf{a}$. The size of the data vector $\mathbf{y}$ is $M \times 1$, that of  $\mathbf{\Psi}$ is $M \times K$, that of the measurement matrix $\mathbf{\Phi}$ is $N \times M$, where $N < M$, and that of the measured vector $\mathbf{x}$ is $N \times 1$. The entries in the random measurement matrix were independent realizations from a standard normal distribution. We recover the underlying image from its compressed measurements, using online (OMP, $\ell_1$), MLD, and RMLD dictionaries. For each case, we present average results from $100$ trial runs, each time with a different measurement matrix. The recovery performance was evaluated for the set of standard images and reported in Table \ref{table:cs_perf1}. Figure \ref{Fig:perf_CS_images} illustrates the recovered images obtained using different dictionaries with $8$ random measurements under noise ($SNR = 15$ dB). We observed that the \textit{MulP} reconstruction using the proposed MLD dictionary resulted in improved recovery performance, at different measurement conditions, when compared to using greedy pursuit (OMP) with the online dictionary. However, both the \textit{MulP} reconstruction for RMLD and $\ell_1$-based reconstruction with the online dictionary perform significantly better than the other two approaches. In particular, the RMLD reconstruction achieves improved recovery at reduced number of measurements ($8$, $16$) and in presence of noise.

\begin{figure*}[tb]	
\begin{minipage}[b]{0.24\linewidth}
 \centering
 \includegraphics[width = 4cm]{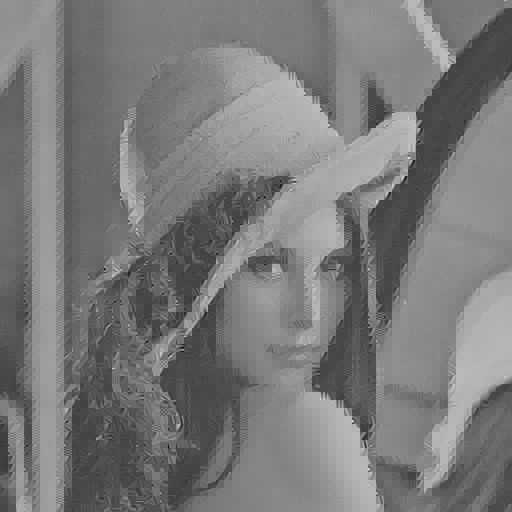}
 \centerline{(a) Online-OMP ($24.73$ dB)}\medskip
\end{minipage}
\hfill
\begin{minipage}[b]{0.24\linewidth}
 \centering
 \includegraphics[width = 4cm]{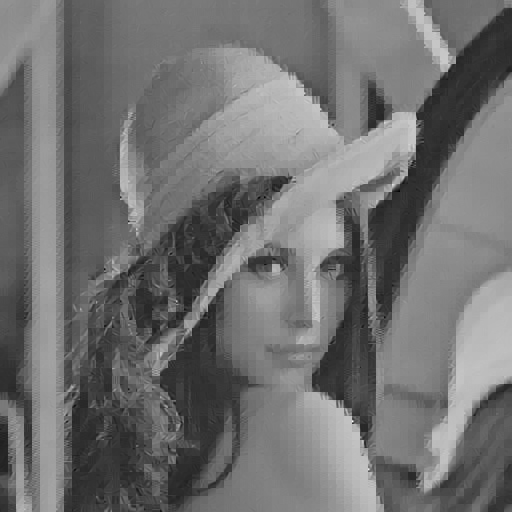}
 \centerline{(b) Online-$\ell_1$ ($25.69$ dB)}\medskip
\end{minipage}
\hfill
\begin{minipage}[b]{0.24\linewidth}
 \centering
 \includegraphics[width = 4cm]{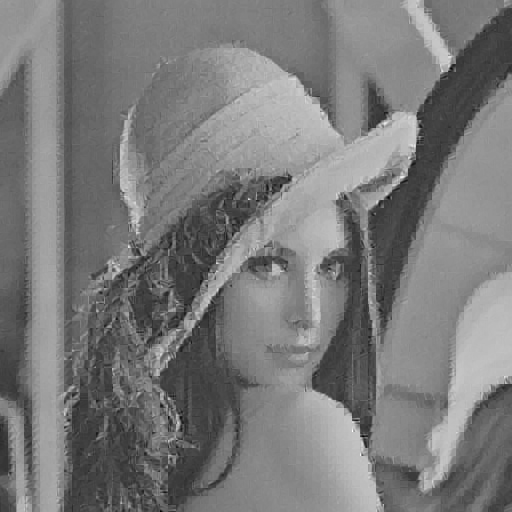}
 \centerline{(c) MLD-MulP ($26.02$ dB)}\medskip
\end{minipage}
\hfill
\begin{minipage}[b]{0.24\linewidth}
 \centering
 \includegraphics[width = 4cm]{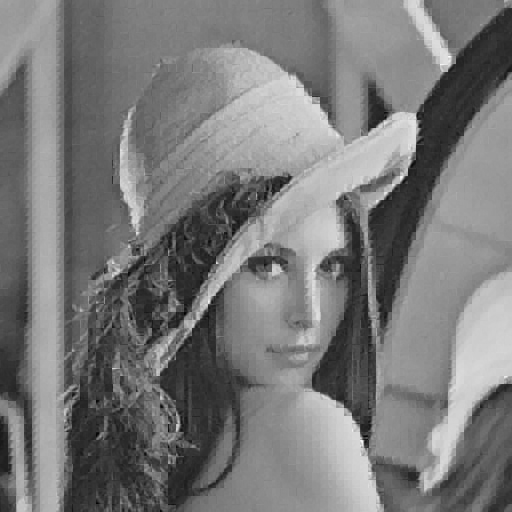}
 \centerline{(d) RMLD-MulP ($27.41$ dB)}\medskip
\end{minipage}
\caption{Compressed recovery of images from random measurements ($N=8$, SNR of measurement process $= 15$dB) using the different dictionaries. In each case the PSNR of the recovered image is also shown.}
\label{Fig:perf_CS_images}
\end{figure*}

\subsection{Application: Subspace Learning}
\label{sec:subspace}
In this section, we evaluate the use of sparse codes obtained with multilevel dictionaries in unsupervised and supervised subspace learning. In particular, we use the locality preserving projections (LPP) \cite{he2003locality} and local discriminant embedding (LDE) \cite{chen2005local} approaches to perform classification on the Forest Covertype dataset. LPP is an unsupervised embedding approach which computes projection directions such that the pairwise distances of the projected training samples in the neighborhood are preserved . Let us define the training data as $\{\mathbf{y}_i|\mathbf{y}_i \in \mathbb{R}^M\}_{i=1}^T$. An undirected graph $G$ is defined, with the training samples as vertices, and the similarity between the neighboring training samples are coded in the affinity matrix $\mathbf{W} \in \mathbb{R}^{T \times T}$. In the proposed setup, we learn a dictionary using the training samples and compute the affinity matrix $\mathbf{W} = |\mathbf{A}^T \mathbf{A}|$, where $\mathbf{A}$ is the matrix of sparse coefficients. Following this, we sparsify $\mathbf{W}$ by retaining only the $\tau$ largest similarities for each sample. Note that this construction is different from the $\ell_1$ graph construction in \cite{l1graph} and computationally efficient.  Let us denote the graph Laplacian as $\mathbf{L} = \mathbf{D}-\mathbf{W}$, where $\mathbf{D}$ is a degree matrix with each diagonal element containing the sum of the corresponding row or column of $\mathbf{W}$. The $d$ projection directions for LPP, $\mathbf{V} \in \mathbb{R}^{M \times d}$, where $d < M$, can be computed by optimizing
\begin{align}
\min_{\text{trace}(\mathbf{V}^T \mathbf{Y} \mathbf{D}\mathbf{Y}^T\mathbf{V}) = \mathbf{I}} \text{trace}(\mathbf{V}^T \mathbf{Y} \mathbf{L}\mathbf{Y}^T\mathbf{V}).
\label{eqn:lpp}
\end{align}Here $\mathbf{Y}$ is a matrix obtained by stacking all data samples as its columns. The embedding for any data sample $\mathbf{z}$ can be obtained as $\mathbf{z} = \mathbf{V}^T \mathbf{y}$. In a supervised setting, we define the intra-class and inter-class affinity matrices, $\mathbf{W}$ and $\mathbf{W}'$ respectively, as
\begin{equation}
w_{ij} = 
\begin{cases}
|\mathbf{a}_i^T \mathbf{a}_j| & \text{if } \pi_i = \pi_j \text{ AND } j \in \mathcal{N}_{\tau}(i), \\
0 & \text{otherwise},
\end{cases}
\label{eqn:ldeintra}
\end{equation}
\begin{equation}
w_{ij}' = 
\begin{cases}
|\mathbf{a}_i^T \mathbf{a}_j| & \text{if } \pi_i \neq \pi_j \text{ AND } j \in \mathcal{N}_{\tau'}(i), \\
0 & \text{otherwise},
\end{cases}
\label{eqn:ldeinter}
\end{equation} where $\pi_i$ is the label of the $i^\text{th}$ training sample, and $\mathcal{N}_{\tau}(i)$ and $\mathcal{N}_{\tau'}(i)$ are the sets that contain the indices of $\tau$ intra-class and $\tau'$ inter-class neighbors of the $i^{\text{th}}$ training sample. The neighbors of a sample $i$ are sorted based on the order of decreasing absolute correlations of their sparse code with $\mathbf{a}_i$. Using these affinity matrices, local discriminant embedding is performed by solving
\begin{equation}
\argmax_{\mathbf{V}} \frac{\text{Tr}[\mathbf{V}^T \mathbf{X}^T \mathbf{L}' \mathbf{X} \mathbf{V}]}{\text{Tr}[\mathbf{V}^T \mathbf{X}^T \mathbf{L} \mathbf{X} \mathbf{V}]}.
\label{eqn:ldeopt}
\end{equation}

For both the subspace learning approaches, we varied the number of training samples between $250$ and $2160$ per class and fixed the embedding dimension $d = 30$. For MLD and RMLD learning, we fixed the number of levels at $32$ and the number of rounds, $R$, for RMLD was fixed at $30$. For comparison, we use learned iterative dictionaries of size $1024$, using $\ell_1$ minimization in the SPAMS toolbox \cite{Mairal_Bach_Ponce_Sapiro_2009} and the Lagrangian dual method (\textit{SC-LD}) \cite{lee2006efficient} . Finally, classification was performed using a simple $1-$nearest neighbor classifier. Table \ref{table:sl_lpp} and Table \ref{table:sl_lde} show the classification accuracies obtained using the different dictionaries, for both the subspace learning approaches. As it can be observed, graphs constructed with the proposed multilevel dictionaries provide more discriminative embeddings compared to the other approaches.

\section{Conclusions}
\label{sec:conclusions}
We presented a multilevel learning algorithm to design generalizable and stable global dictionaries for sparse representations. The proposed algorithm uses multiple levels of $1-$D subspace clustering to learn dictionaries. We also proposed a method to infer the number of atoms in each level, and provided an ensemble learning approach to create robust dictionaries. We proved that the learning algorithm converges, exhibits energy hierarchy, and is also generalizable and stable. Finally, we demonstrated the superior performance of MLD in applications such as compressive sensing and subspace learning. Future research could include providing an online framework for MLD that can work with streaming data, and also developing hierarchical dictionaries that are optimized for robust penalties on reconstruction error.

\begin{table}[tb]
\renewcommand{\arraystretch}{1.3}
  \centering
\caption{Unsupervised subspace Learning - Classification accuracies with a 1-NN classifier.}
\begin{tabular}{|c||c|c|c|c|}
\hline
\textbf{\# Train}&\multicolumn{4}{c|}{\textbf{Graph Construction Approach}}\\
\cline{2-5}
\textbf{Per Class}&\textbf{LPP}&\textbf{SC-LD}&\textbf{SC-MLD}&\textbf{SC-RMLD}\\
\hline
\hline
250&57.11&57.5&58.2&\textbf{58.9}\\
500&58.8&59.9&61.35&\textbf{62.58}\\
1000&66.33&67.6&68.94&\textbf{69.91}\\
1500&70.16&71.32&73.65&\textbf{74.38}\\
2160&74.39&75.8&78.26&\textbf{78.84}\\
\hline
\end{tabular}
\label{table:sl_lpp}
\end{table}

\begin{table}[tb]
\renewcommand{\arraystretch}{1.3}
  \centering
\caption{Supervised subspace Learning - Classification accuracies with a 1-NN classifier.}
\begin{tabular}{|c||c|c|c|c|}
\hline
\textbf{\# Train}&\multicolumn{4}{c|}{\textbf{Graph Construction Approach}}\\
\cline{2-5}
\textbf{Per Class}&\textbf{LDE}&\textbf{SC-LD}&\textbf{SC-MLD}&\textbf{SC-RMLD}\\
\hline
\hline
250&59.23&59.1&59.3&\textbf{59.6}\\
500&60.4&60.8&61.9&\textbf{62.7}\\
1000&68.1&68.71&69.6&\textbf{70.43}\\
1500&72.9&73.5&74.41&\textbf{75.09}\\
2160&77.3&78.07&79.53&\textbf{80.01}\\
\hline
\end{tabular}
\label{table:sl_lde}
\end{table}

\bibliographystyle{IEEEtran}
\bibliography{refs}
\end{document}